\documentclass[pmlr]{jmlr}

\jmlrvolume{329}
\jmlryear{2026}
\jmlrworkshop{Conformal and Probabilistic Prediction with Applications}
\jmlrproceedings{PMLR}{Proceedings of Machine Learning Research}

\makeatletter
\renewcommand\ps@jmlrtps{%
  \let\@mkboth\@gobbletwo
  \def\@oddhead{}%
  \let\@evenhead\@oddhead
  \def\@oddfoot{\@titlefoot}%
  \let\@evenfoot\@oddfoot
}
\renewcommand*{\@titlefoot}{\scriptsize\copyright\space\@jmlryear\space
  A.~Koukorinis \& R.~Silva.\hfill}
\makeatother

\usepackage[T1]{fontenc}  
\usepackage{lmodern}      
\usepackage{mathtools}
\usepackage{enumitem}
\usepackage{booktabs}
\usepackage[expansion=false]{microtype}
\usepackage{caption}
\usepackage{tcolorbox}

\newtheorem{assumption}{A}

\newcommand{\R}{\mathbb{R}}
\newcommand{\E}{\mathbb{E}}
\newcommand{\Prob}{\mathbb{P}}
\newcommand{\Var}{\mathrm{Var}}
\newcommand{\Cov}{\mathrm{Cov}}
\newcommand{\ind}{\mathbf{1}}
\newcommand{\calF}{\mathcal{F}}
\newcommand{\calB}{\mathcal{B}}
\newcommand{\calG}{\mathcal{G}}
\newcommand{\calX}{\mathcal{X}}
\newcommand{\calY}{\mathcal{Y}}
\newcommand{\dTV}{d_{\mathrm{TV}}}

\newcommand{\RXKurtosis}{105.41}
\newcommand{\RXSqACOne}{0.012}
\newcommand{\RXSqACFive}{0.006}
\newcommand{\RXSqACTen}{0.006}

\newcommand{\GARCHOneOmega}{5.43e-11}
\newcommand{\GARCHOneAlpha}{0.0500}
\newcommand{\GARCHOneBeta}{0.9300}
\newcommand{\GARCHOnePersistence}{0.9800}

\newcommand{\GARCHTwoOmega}{5.43e-11}
\newcommand{\GARCHTwoAlphaOne}{0.0250}
\newcommand{\GARCHTwoAlphaTwo}{0.0250}
\newcommand{\GARCHTwoBeta}{0.9300}
\newcommand{\GARCHTwoPersistence}{0.9800}
\newcommand{\StudentTDOF}{2.50}
\newcommand{\EmpQOne}{-2.502}
\newcommand{\EmpQFive}{-1.488}
\newcommand{\EmpQNinetyFive}{1.487}
\newcommand{\EmpQNinetyNine}{2.509}
\newcommand{\StudentTQOne}{-5.353}
\newcommand{\StudentTQFive}{-2.558}
\newcommand{\StudentTQNinetyFive}{2.558}
\newcommand{\StudentTQNinetyNine}{5.353}

\newcommand{\CovariateRho}{0.03}

\title{\textbf{Doubly Robust Adaptive Conformal Inference\\for Causal Effects Under Temporal Dependence}}

\author{Andreas Koukorinis\\
Department of Computer Science\\
University College London\\
\texttt{andreas.koukorinis.12@ucl.ac.uk}
\AND
Ricardo Silva\\
Department of Statistical Science\\
University College London\\
\texttt{ricardo.silva@ucl.ac.uk}}

\begin{document}

\editor{Khuong An Nguyen, Zhiyuan Luo,
  Harris Papadopoulos, Tuwe L\"{o}fstr\"{o}m,
  Lars Carlsson and Henrik Bostr\"{o}m}

\maketitle

\begin{abstract}
We propose \emph{doubly robust adaptive conformal inference} (DR-ACI), which constructs prediction intervals for doubly robust pseudo-outcomes under temporal dependence.  Calibration targets the pseudo-outcome $\psi^{\mathrm{DR}}_t$; under estimator consistency, this yields asymptotically conservative CATE containment (Corollary~\ref{cor:cate}).  Temporal block cross-fitting preserves switch-coefficient mixing bounds and the DML product-bias rate up to an explicit coupling remainder.  The resulting coverage guarantee decomposes into three terms: mixing gap, nuisance-bias tax, and adaptation rate.  Under geometric $\beta$-mixing, the rate matches the exchangeable case.  VS-DR-ACI produces 63\% narrower intervals than split conformal via variance standardisation.  Under combined dependence and drift, VS-DR-ACI maintains valid coverage (89.9\%) with stable width, while DML-based confidence intervals lose 35pp after drift onset.  We apply the method to Nasdaq's Dynamic M-ELO rollout.  The mixing bound overestimates coverage loss in practice; deriving a tighter bound is future work.
\end{abstract}

\begin{keywords}
conformal prediction, doubly robust estimation, $\beta$-mixing, treatment effects, market microstructure
\end{keywords}

\medskip
\noindent\textbf{Code availability:} Simulation code is available at \url{https://github.com/rockandrolla13/draci}.

\bigskip

\section{Introduction}\label{sec:intro}

Quantifying uncertainty around heterogeneous treatment effects requires prediction intervals that are valid observation-by-observation, not just on average.  Standard conformal prediction \citep{vovk2005algorithmic}, including split conformal \citep{lei2018conformal}, constructs intervals for observable outcomes $Y$; our target is the conditional average treatment effect $\tau(X)$, which is latent.  We access $\tau(X)$ through doubly robust pseudo-outcomes and calibrate on these constructed scores rather than raw residuals.  Under temporal dependence, common in financial and epidemiological data, na\"ive conformal prediction fails because exchangeability no longer holds, and na\"ive causal intervals (asymptotic, bootstrap, Bayesian) lose their distribution-free calibration guarantee.  We use conformal prediction rather than confidence intervals because conformal methods are distribution-free and finite-sample valid without correct model specification or asymptotics, neither of which machine learning estimators under dependence can easily provide.  The goal is finite-sample coverage control for individual treatment effect estimates without assuming independent observations.

Doubly robust pseudo-outcomes are attractive for this purpose because they absorb nuisance estimation error orthogonally: the bias of $\psi_t^{\mathrm{DR}}$ as an estimate of $\tau(X_t)$ is bounded by the \emph{product} $\|\hat{e}-e\|\cdot\|\hat\mu-\mu\|$, not the sum, so consistency requires only that one nuisance model be well-specified \citep{chernozhukov2018dml}.  However, once data are temporally dependent and nuisance models must be fit on parts of the time series, validity of conformal calibration is no longer automatic.  The main technical obstacle is that standard cross-fitting independence---training on fold $k$, calibrating on fold $j \neq k$---fails under $\beta$-mixing: adjacent observations remain correlated, and temporally shuffled folds introduce look-ahead bias.  Simply applying existing conformal results to DR pseudo-outcomes does not yield valid coverage bounds; the coupling between nuisance estimation error and temporal dependence must be controlled explicitly.

The paper shows that \emph{temporal block cross-fitting with guard bands}---partitioning the time series into contiguous blocks, discarding guard bands of size $g$ adjacent to each held-out block, training nuisance models on the remainder, and calibrating on the held-out block---preserves two properties needed for feasible conformal calibration under $\beta$-mixing:
\begin{enumerate}[label=(\alph*),leftmargin=2em,itemsep=1pt]
\item \textbf{Switch-coefficient control for DR scores.}  Conformity scores constructed from DR pseudo-outcomes inherit the switch-coefficient bound of the underlying process, yielding coverage gap $\min_\tau\{\tau/T + 2\beta(\tau)\}$ (Lemma~\ref{lem:switch}).
\item \textbf{DML product-bias bound under $\beta$-mixing.}  The product-bias rate $O(\|\hat{e}-e\|_2\cdot\|\hat\mu-\mu\|_2)$ extends to temporally dependent data with an explicit coupling remainder $O(\beta(g)^{\delta/(2+\delta)})$ via Rio's inequality, where $g$ is the guard band size (Lemma~\ref{lem:productbias}).  This is the main technical contribution.
\end{enumerate}
Adaptive conformal inference (ACI) is layered on top, providing the deterministic long-run coverage guarantee of \citet{gibbs2021aci}.  Our contribution is not the existence of these ingredients individually, but a proof that they compose under temporal dependence.

\begin{tcolorbox}[colback=gray!5,colframe=gray!75,title={Scope of Guarantees}]
\begin{itemize}[leftmargin=1.5em,itemsep=2pt]
\item \textbf{What is covered:} DR pseudo-outcomes $\psi_t^{\mathrm{DR}}$ under temporal dependence.
\item \textbf{What is guaranteed:} Coverage-gap decomposition (Theorem~\ref{thm:main}) with explicit rates plus asymptotic validity under stated mixing and nuisance-rate conditions.
\item \textbf{What links to CATEs:} Asymptotic CATE containment when the point estimator $\hat{\tau}(X_t)$ is consistent (Corollary~\ref{cor:cate}).  This is a consistency-driven implication, not a calibrated coverage guarantee.
\item \textbf{What is not proved:} Finite-sample $(1-\alpha)$-coverage for the latent CATE $\tau(X_t)$.  The conformal calibration targets pseudo-outcome deviations, not CATE estimation error.
\end{itemize}
\end{tcolorbox}

\paragraph{Related work.}
Prior work addresses these axes separately.  \citet{lei2021conformal} combine doubly robust estimation with conformal inference but assume exchangeability.  \citet{barber2025timeseries} develop conformal prediction under $\beta$-mixing but do not incorporate causal identification.  \citet{gibbs2021aci} provide adaptive recalibration but analyse neither doubly robust scores nor mixing rates.  \citet{lee2025kowcpi} combine mixing and adaptation but not causal inference; \citet{farina2025survival,sesia2025survival} develop DR conformal for survival but assume exchangeability.  DR-ACI occupies the intersection (doubly robust scores, mixing-valid coverage, online adaptation) not present in prior work.

Stock-level effect uncertainty matters operationally: market design changes affect execution quality heterogeneously across securities.  Temporal dependence is intrinsic to microstructure data.  We apply DR-ACI to Nasdaq's Dynamic M-ELO rollout, producing stock-level prediction intervals for hidden-order execution quality.

\paragraph{Contributions.}
\begin{itemize}[leftmargin=2em,itemsep=1pt]
\item \emph{Main result:} Three-term coverage decomposition (Theorem~\ref{thm:main}, Section~\ref{sec:theory}) with mixing gap $\min_\tau\{\tau/T + 2\beta(\tau)\}$, nuisance-bias tax $\|\hat{e}-e\|_2 \cdot \|\hat\mu-\mu\|_2$, and adaptation rate $O(T^{-1/2})$---the first coverage guarantee for DR pseudo-outcomes under $\beta$-mixing.
\item Switch-coefficient inheritance for DR conformity scores (Lemma~\ref{lem:switch}).
\item DML product-bias preservation under temporal block cross-fitting with explicit coupling remainder (Lemma~\ref{lem:productbias}).
\item VS-DR-ACI reduces interval width 63\% vs.\ split conformal at valid coverage (Section~\ref{sec:simulations}).
\end{itemize}

\section{Setup and Inferential Targets}\label{sec:setup}

This section defines the observed, constructed, latent, and reported objects in DR-ACI and distinguishes the calibration target from the scientific target.

\subsection{Data and Causal Estimand}\label{sec:data}

We observe a temporally dependent process $\{(X_t, W_t, Y_t)\}_{t=1}^T$, where $X_t \in \calX \subseteq \R^p$ is a covariate vector, $W_t \in \{0,1\}$ is a binary treatment indicator, and $Y_t \in \calY \subseteq \R$ is the realized outcome. Each unit has potential outcomes $Y_t(0)$ and $Y_t(1)$, with $Y_t = W_t \cdot Y_t(1) + (1-W_t) \cdot Y_t(0)$.

The target of scientific interest is the conditional average treatment effect:
\begin{equation}\label{eq:cate}
\tau(x) = \E[Y_t(1) - Y_t(0) \mid X_t = x].
\end{equation}
This function is \emph{latent}: it cannot be computed from observed data because both potential outcomes are never observed for the same unit. Identification requires two standard assumptions.

\begin{assumption}[Unconfoundedness]\label{ass:unconf}
$W_t \perp (Y_t(0), Y_t(1)) \mid X_t$ for all $t$.
\end{assumption}

\begin{assumption}[Overlap]\label{ass:overlap}
There exists $\eta > 0$ such that $\eta \leq e(x) \leq 1 - \eta$ for all $x \in \calX$, where $e(x) = \Prob(W_t = 1 \mid X_t = x)$.
\end{assumption}

\subsection{Objects}\label{sec:objects}

The method involves six objects, summarized below; their precise definitions appear in Sections~\ref{sec:data}, \ref{sec:targets}, and \ref{sec:method}.

\begin{center}
\renewcommand{\arraystretch}{1.2}
\begin{tabular}{@{}lll@{}}
\toprule
\textbf{Object} & \textbf{Status} & \textbf{Role} \\
\midrule
$(X_t, W_t, Y_t)$ & Observed & Raw data \\
$\tau(X_t)$ & Latent & Scientific target \\
$\psi_t^{\mathrm{DR}}$ & Constructed & Primary conformal target \\
$\hat\tau(X_t)$ & Estimated & Centering object \\
$s_t$ & Computed & Conformity measure \\
$\hat{C}_t$ & Reported & Final output \\
\bottomrule
\end{tabular}
\end{center}

\subsection{Inferential Targets}\label{sec:targets}

Conformal prediction on $Y$ would produce valid intervals for the outcome, not for the treatment effect.  The CATE $\tau(X_t) = \E[Y(1) - Y(0) \mid X_t]$ is latent---never directly observed---so conformal calibration must target an observable proxy.  The doubly robust pseudo-outcome
\begin{equation}\label{eq:dr}
\psi_t^{\mathrm{DR}} = \hat\mu_1(X_t) - \hat\mu_0(X_t) + \frac{W_t(Y_t - \hat\mu_1(X_t))}{\hat{e}(X_t)} - \frac{(1-W_t)(Y_t - \hat\mu_0(X_t))}{1 - \hat{e}(X_t)}
\end{equation}
satisfies $\E[\psi^{\mathrm{DR}} \mid X] = \tau(X)$ under correct nuisance specification, where $\hat{e}$ estimates the propensity score and $\hat\mu_w$ the conditional outcome.  Writing $\psi^{\mathrm{DR}} = \tau(X) + \xi$ with $\E[\xi \mid X] = 0$, the conformal score $s = |\psi^{\mathrm{DR}} - \hat{\tau}|$ measures deviation of this noisy proxy from the point estimate.  The resulting guarantee $\Pr(\psi^{\mathrm{DR}} \in \hat{C}) \geq 1-\alpha$ applies directly to the observable pseudo-outcome; CATE coverage follows indirectly since $\psi^{\mathrm{DR}}$ is centered at $\tau(X)$.

Direct finite-sample conditional coverage for latent $\tau(X_t)$ is impossible without distributional assumptions \citep{vovk2012conditional}.

Because the conformal procedure explicitly calibrates on the doubly robust pseudo-outcome $\psi_t^{\mathrm{DR}}$---a random variable subject to realization-specific noise $\xi_t$---we formally construct a \emph{prediction interval}.  However, under estimator consistency (Assumption~\ref{ass:cate}), this prediction interval for the pseudo-outcome yields an asymptotically conservative containment interpretation for the latent CATE $\tau(X_t)$: the interval over-covers $\tau(X_t)$ because it was calibrated for the wider target $\psi_t^{\mathrm{DR}} = \tau(X_t) + \xi_t$.  This duality---prediction interval for the observable, conservative containment for the latent parameter---is formalized in Corollary~\ref{cor:cate}.

\textbf{Primary target.} The conformal procedure directly calibrates coverage for the pseudo-outcome $\psi_t^{\mathrm{DR}}$ under temporal dependence, with explicit rate bounds established in Section~\ref{sec:theory}.

\textbf{Secondary target.} The latent CATE $\tau(X_t)$ is the scientific quantity of interest. Under additional consistency and regularity conditions on the point estimator, the reported intervals also yield an asymptotic containment statement for the latent CATE. This interpretation is derived rather than direct, requires the assumption below, and is formalized in Section~\ref{sec:theory}.

\begin{assumption}[CATE estimation consistency]\label{ass:cate}
$\|\hat\tau - \tau\|_{L^2} = o_p(1)$.
\end{assumption}

If $\hat\tau$ is inconsistent, the interval may cover $\psi_t^{\mathrm{DR}}$ at the nominal rate while excluding $\tau(X_t)$, even under additional conditions.

\subsection{Sources of Uncertainty}\label{sec:uncertainty}

Three sources affect the reported interval: (i) residual variability of the pseudo-outcome around the latent effect, (ii) nuisance estimation error in $\hat{e}$ and $\hat\mu$, and (iii) CATE estimation error in $\hat\tau$. The theory decomposes the coverage gap into terms corresponding to these sources.

\subsection{Dependence and Regularity Assumptions}\label{sec:dependence}

\begin{definition}[$\beta$-mixing]\label{def:betamixing}
For a stationary process $\{Z_t\}_{t \geq 1}$, the $\beta$-mixing coefficient at lag $k$ is
\[
\beta(k) = \sup_{j \geq 1} \dTV\bigl(\mathcal{L}(Z_1^j, Z_{j+k+1}^\infty),\; \mathcal{L}(Z_1^j) \otimes \mathcal{L}(Z_{j+k+1}^\infty)\bigr),
\]
where $Z_a^b = (Z_a, Z_{a+1}, \ldots, Z_b)$ denotes the sub-sequence from index $a$ to $b$, $Z_{j+k+1}^\infty$ is the infinite tail starting at index $j+k+1$, and $\dTV$ is total variation distance.  Intuitively, $\beta(k)$ measures how close the ``past'' $(Z_1, \ldots, Z_j)$ and ``future'' $(Z_{j+k+1}, Z_{j+k+2}, \ldots)$ are to independence when separated by a gap of $k$ time steps.
\end{definition}

\begin{assumption}[$\beta$-mixing]\label{ass:mixing}
The process $\{(X_t, W_t, Y_t(0), Y_t(1))\}_{t \geq 1}$ is strictly stationary and $\beta$-mixing with $\beta(k) \leq C_\beta k^{-r}$ for some $r > 1$.
\end{assumption}

This assumption includes common weakly dependent time-series models (AR, ARMA, GARCH).  \textbf{Important:} The strict stationarity requirement excludes structural breaks, regime changes, and covariate drift.  Regime D experiments in Section~\ref{sec:simulations} violate this assumption and are empirical robustness evidence only, not covered by Theorem~\ref{thm:main}.

\begin{assumption}[Nuisance convergence]\label{ass:nuisance}
$\|\hat{e}_k - e\|_{L^2} = O_p(T^{-\zeta_e})$ and $\|\hat\mu_{w,k} - \mu_w\|_{L^2} = O_p(T^{-\zeta_\mu})$ uniformly over blocks $k$, with $\zeta_e + \zeta_\mu > 1/2$.
\end{assumption}

Standard nonparametric rates under $\beta$-mixing are $\zeta = p/(2p+d)$ for $p$-smooth functions in $d$ dimensions \citep{gyorfi2002mixing}, achievable by kernel or local polynomial estimators.  For tree-based learners (random forests, gradient boosting), rates of $T^{-1/3}$ to $T^{-1/4}$ are typical under smoothness assumptions \citep{wager2018forests}, satisfying the product-bias condition when both nuisance models achieve similar rates.

\begin{assumption}[Moment condition]\label{ass:moment}
$\E[|\psi_t^{\mathrm{DR}}|^{2+\delta}] < \infty$ for some $\delta > 0$.
\end{assumption}

Assumptions~\ref{ass:unconf}--\ref{ass:overlap} identify the causal estimand. Assumptions~\ref{ass:mixing}--\ref{ass:moment} support the coverage theory. Assumption~\ref{ass:cate} is additionally required for asymptotic CATE interpretation.

\section{The DR-ACI Method}\label{sec:method}

\subsection{Guard Bands for Approximate Independence}\label{sec:guardbands}

Standard cross-fitting assumes independence between training and calibration folds---an assumption violated under time series.  When training on observations $1,\ldots,t-1$ and calibrating on observation $t$, the nearest training point is one step away, and residual correlation persists.

We employ \emph{guard bands}---gaps of $g$ observations between training and calibration blocks---following the $h$-block cross-validation framework of \citet{burman1994crossvalidatory} and \citet{racine2000consistent}, adapted to the conformal prediction setting.  This approach parallels the ``neighbours-left-out'' cross-fitting of \citet{semenova2023inference} for DML under weak dependence; see also \citet{bian2024conformal} for related temporal separation in Markovian conformal prediction.  When calibrating on block $k$, we exclude the $g$ adjacent observations on each side from the training set, ensuring the nearest training observation is at least $g$ steps from any calibration point.  Under $\beta$-mixing, residual dependence between training and calibration sets is $O(\beta(g)^{\delta/(2+\delta)})$---the coupling error made explicit in Lemma~\ref{lem:productbias} via Rio's inequality \citep{rio2017mixing}, where $\delta > 0$ is the moment exponent in Assumption~\ref{ass:moment}.  Setting $g = b$ (guard band equals block size) balances the bias-variance trade-off: larger guards reduce coupling error but discard more training data.

Guard bands are what allow DML-style product-bias guarantees to extend to temporal dependence.  Figure~\ref{fig:temporal_blocks} illustrates the construction.

\begin{figure}[t]
\centering
\includegraphics[width=0.88\textwidth]{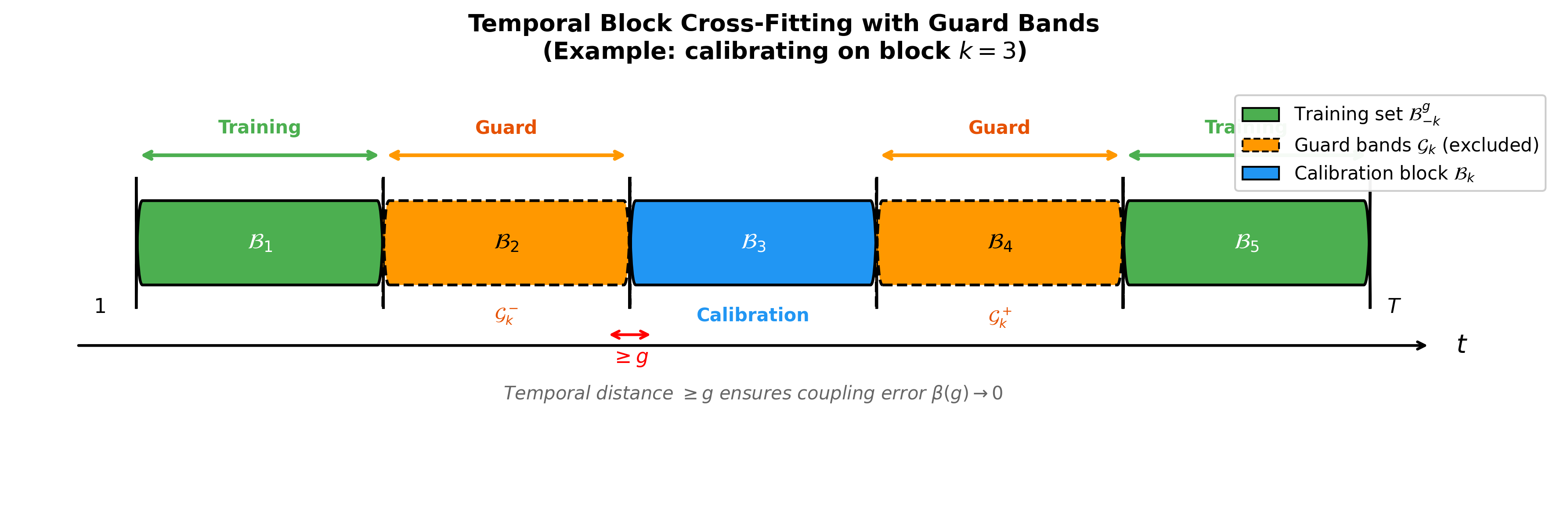}
\caption{Temporal block cross-fitting with guard bands. When calibrating on block $\mathcal{B}_k$ (blue), guard bands $\mathcal{G}_k$ (orange, dashed) are excluded from training to ensure temporal separation $\geq g$. Training uses only $\mathcal{B}_{-k}^g$ (green). The coupling error $\beta(g)$ decays with guard band size.}
\label{fig:temporal_blocks}
\end{figure}

\subsection{Algorithm}\label{sec:algorithm}

DR-ACI proceeds in three stages: temporal block cross-fitting for nuisance estimation, DR pseudo-outcome construction, and online conformal calibration.  The conformity score measuring deviation from the point estimate is
\begin{equation}\label{eq:score}
s_t = \bigl|\psi_t^{\mathrm{DR}} - \hat\tau(X_t)\bigr|.
\end{equation}

\begin{algorithm2e}[t]
\DontPrintSemicolon
\caption{DR-ACI: Doubly Robust Adaptive Conformal Inference}\label{alg:draci}
\KwIn{Data $\{(X_t, W_t, Y_t)\}_{t=1}^T$; number of blocks $K$; guard band size $g$; miscoverage level $\alpha$; learning rate $\gamma > 0$; CATE estimator $\hat\tau$}
\KwOut{Prediction intervals $\{\hat{C}_t^{\mathrm{DR-ACI}}\}_{t=1}^T$}
\textbf{Stage 1: Temporal block cross-fitting with guard bands}\;
Partition $\{1,\ldots,T\}$ into $K$ contiguous blocks $\calB_1,\ldots,\calB_K$ of size $b = T/K$\;
\For{$k = 1,\ldots,K$}{
  Define guard band $\calG_k = g$ observations before and after $\calB_k$\;
  Train nuisance models $(\hat{e}_k, \hat\mu_{0,k}, \hat\mu_{1,k})$ on $\calB_{-k}^g = \{1,\ldots,T\} \setminus (\calB_k \cup \calG_k)$\;
}
\textit{(Guard bands excluded from training only; all $T$ observations enter calibration.)}\;
\textbf{Stage 2: DR pseudo-outcome construction}\;
\For{$t = 1,\ldots,T$}{
  Let $k(t)$ be the block containing $t$\;
  Compute $\psi_t^{\mathrm{DR}}$ via \eqref{eq:dr} using $(\hat{e}_{k(t)}, \hat\mu_{0,k(t)}, \hat\mu_{1,k(t)})$\;
}
\textbf{Stage 3: Online conformal calibration}\;
Initialise $\alpha_1 = \alpha$\;
\For{$t = 1,\ldots,T$}{
  Compute conformity score $s_t = |\psi_t^{\mathrm{DR}} - \hat\tau(X_t)|$\;
  Set $\hat{q}_t = \mathrm{Quantile}_{1-\alpha_t}(s_1,\ldots,s_{t-1})$\;
  Form prediction interval $\hat{C}_t^{\mathrm{DR-ACI}} = [\hat\tau(X_t) - \hat{q}_t,\; \hat\tau(X_t) + \hat{q}_t]$\;
  \textbf{Update:} $\alpha_{t+1} = \alpha_t + \gamma(\alpha - \ind\{\psi_t^{\mathrm{DR}} \notin \hat{C}_t^{\mathrm{DR-ACI}}\})$\;
}
\end{algorithm2e}

\textbf{Block size and guard band selection.}  The block size $b$ and guard band size $g$ trade off nuisance estimation quality against mixing decorrelation; see Lemma~\ref{lem:productbias} for the formal trade-off.  In practice, set $g = b$ (guard band equals block size); set $g \geq L$ when nuisance models use $L$ lagged features.  The coverage guarantee (Theorem~\ref{thm:main}) holds for any $g \geq 1$.  Sensitivity analysis in Table~\ref{tab:guardsens} (Appendix~\ref{app:guardsens}) confirms coverage remains valid across guard band sizes.

\textbf{Learning rate.}  We follow \citet{gibbs2021aci} and set $\gamma = 0.005$ in all experiments.  The long-run coverage guarantee \eqref{eq:aci} holds for any $\gamma > 0$.

\subsection{Variance-Standardised DR-ACI (VS-DR-ACI)}\label{sec:vsdraci}

DR-ACI targets $1-\alpha$ coverage of $\psi_t^{\mathrm{DR}} = \tau(X_t) + \xi_t$, where $\xi_t$ is mean-zero noise from the DR construction.  Because the calibration quantile absorbs $\Var(\xi_t)$, intervals are wider than necessary for containing the latent CATE alone.  Under Assumption~\ref{ass:cate} (consistent $\hat\tau$), the centering error $|\hat\tau(X_t) - \tau(X_t)|$ shrinks to zero while the interval half-width $\hat{q}_t$ remains positive (tracking pseudo-outcome variability), so CATE containment probability approaches 1.  When nuisance estimation is poor, $\Var(\xi)$ dominates and intervals are unnecessarily wide.

To mitigate this, we define the variance-standardised conformity score:
\begin{equation}\label{eq:vsscore}
s_t^{\mathrm{VS}} = \frac{|\psi_t^{\mathrm{DR}} - \hat\tau(X_t)|}{\hat\sigma_\xi(X_t)},
\end{equation}
where $\hat\sigma_\xi(x)$ estimates the local standard deviation of the DR noise.  From the influence function of $\psi^{\mathrm{DR}}$, the noise variance is
\begin{equation}\label{eq:noisevar}
\sigma_\xi^2(x) = \frac{\Var(Y \mid X=x, W=1)}{e(x)} + \frac{\Var(Y \mid X=x, W=0)}{1-e(x)},
\end{equation}
which we estimate by plugging in $\hat{e}$ and local variance estimates.  The VS score equalises the score distribution across covariate strata, producing intervals whose width adapts to local noise.  Algorithm~\ref{alg:draci} is modified at the score and interval steps: replace $s_t$ with $s_t^{\mathrm{VS}}$ and replace the interval construction with $\hat{C}_t^{\mathrm{VS}} = [\hat\tau(X_t) - \hat{q}_t \cdot \hat\sigma_\xi(X_t),\; \hat\tau(X_t) + \hat{q}_t \cdot \hat\sigma_\xi(X_t)]$.

Under Assumption~\ref{ass:overlap}, VS scores are uniformly bounded (Proposition~\ref{prop:vsstability}), so the ACI quantile cannot spiral regardless of drift in $P(X_t)$.

\textbf{Computational complexity.}  Total cost is $O(Tp\log T)$ dominated by nuisance fitting; $T{=}2000$, $K{=}5$ runs in under 2 seconds.

\section{Theory}\label{sec:theory}

\subsection{Main Result}

\begin{theorem}[Coverage guarantee for DR-ACI]\label{thm:main}
Under Assumptions~\ref{ass:mixing}--\ref{ass:moment}, the DR-ACI prediction intervals satisfy:

\emph{(i) Long-run coverage.}
\[
\liminf_{T \to \infty} \frac{1}{T}\sum_{t=1}^T \ind\!\bigl\{\psi_t^{\mathrm{DR}} \in \hat{C}_t^{\mathrm{DR-ACI}}\bigr\} \geq 1 - \alpha \quad \text{a.s.}
\]

\emph{(ii) Coverage gap decomposition.}  For each $T \geq 1$,
\begin{align}\label{eq:mainbound}
\Biggl|\frac{1}{T}\sum_{t=1}^T \Prob\bigl(\psi_t^{\mathrm{DR}} \notin \hat{C}_t\bigr) - \alpha\Biggr| &\leq \underbrace{\min_{\tau}\bigl\{\tfrac{\tau}{T} + 2\beta(\tau)\bigr\}}_{\text{mixing gap}} \notag\\
&\quad + \underbrace{\|\hat{e}-e\|_2 \|\hat\mu-\mu\|_2 + O(\beta(g)^{\delta/(2+\delta)})}_{\text{nuisance-bias tax}} + \underbrace{O(T^{-1/2})}_{\text{adaptation}}.
\end{align}
The coupling remainder $\beta(g)$ arises from guard-band decoupling (Lemma~\ref{lem:productbias}); for $g = T^{1/(r+1)}$ with $r \geq 1$, $\beta(g) = O(T^{-r/(r+1)}) = o(T^{-1/2})$ and is absorbed by the adaptation rate.
\end{theorem}

\noindent
Part (i) is an immediate consequence of the ACI deterministic bound \eqref{eq:aci}, which holds for \emph{any} score sequence \citep{gibbs2021aci}.  The contribution of DR-ACI begins with part (ii), where the three terms isolate distinct sources of coverage gap.  The proof of part (ii) combines three modular lemmas.

\begin{remark}[Theory--experiment gap in the mixing bound]
The switch coefficient bound $\min_\tau\{\tau/T + 2\beta(\tau)\}$
in \eqref{eq:mainbound} is tight for adversarial sequences
but loose for DR scores.
At $\rho=0.99$, the bound predicts coverage could fall to 67\%,
yet VS-DR-ACI achieves 90\% in simulations (Section~\ref{sec:simulations}).
Why the gap?  DR scores are smooth perturbations of $X_t$: replacing
one score shifts the distribution by at most
$2 \cdot \mathrm{TV}(\xi_t, \xi_t')$, where $\xi_t$ is
the DR noise---well below the worst-case $2\beta(\tau)$
unless $\xi_t$ has heavy tails.
A tighter bound for smooth score functions remains open.  One path forward is to exploit the Lipschitz structure of DR scores: if the DR noise $\xi_t$ has a density with bounded Lipschitz constant $L_\xi$, then by the Kantorovich--Rubinstein theorem $\mathrm{TV}(\xi_t, \xi_t') \leq L_\xi \cdot \E[|\xi_t - \xi_t'|]$, giving a switch coefficient bound that decays with the smoothness of the score distribution rather than the worst-case $\beta(\tau)$.
\end{remark}

\begin{conjecture}[Smooth DR mixing bound]
Under Assumptions~4--6, suppose additionally that the conditional DR noise density
$f_{\xi}(\cdot\mid X_t)$ has Lipschitz constant bounded by $\bar{L}<\infty$ uniformly in
$X_t$. Then $\Psi_{k,\tau}(S) \le C\bar{L}\cdot\beta(\tau)^{\delta/(2+\delta)}$ for all
$k\le T-\tau$, improving the mixing-gap term in Theorem~2 from
$\min_\tau\{\tau/T + 2\beta(\tau)\}$ to $\min_\tau\{\tau/T + 2C\bar{L}\,\beta(\tau)^{\delta/(2+\delta)}\}$.
Under geometric $\beta$-mixing, the rate remains $O(\log T/T)$ but the prefactor
$\bar{L}<1$ explains the gap between the theoretical bound ($67\%$ at $\rho=0.99$) and
observed coverage ($90\%$). The $L^1$ coupling step requires Assumption~6 via Jensen's
inequality applied to Rio's $L^{2+\delta}$ coupling (Rio, 2017, Thm.~5.1).
\end{conjecture}

\begin{corollary}[Rate for financial data]\label{cor:financial}
Under GARCH(1,1) with geometric $\beta$-mixing \citep{carrasco2002mixing}, the mixing gap is $O(e^{-c\tau})$ for some $c > 0$, so $\min_\tau\{\tau/T + 2\beta(\tau)\} = O(\log T / T)$.  Under Assumption~\ref{ass:nuisance} with $\zeta_e + \zeta_\mu > 1/2$, the nuisance tax is $o(T^{-1/2})$.  The dominant term is the adaptation rate $O(T^{-1/2})$---\emph{the same rate as under exchangeability}.
\end{corollary}

\begin{corollary}[Asymptotic CATE containment]\label{cor:cate}
Decompose the DR pseudo-outcome as $\psi_t^{\mathrm{DR}} = \tau(X_t) + \xi_t$, where $\xi_t$ is the residual noise.  Under Assumptions~\ref{ass:mixing}--\ref{ass:nuisance}:

\emph{(i) Conditional unbiasedness.}  The residual satisfies
\[
\bigl|\E[\xi_t \mid X_t, \calF_{\mathrm{train}}]\bigr| \leq C \cdot \|\hat{e} - e\|_\infty \cdot \|\hat\mu - \mu\|_\infty,
\]
which is $o(1)$ under Assumption~\ref{ass:nuisance}.

\emph{(ii) Asymptotic CATE containment.}  Let $\hat{C}_t = [\hat\tau(X_t) - \hat{q}_t, \hat\tau(X_t) + \hat{q}_t]$ be the DR-ACI interval centered at point estimator $\hat\tau(X_t)$ with half-width $\hat{q}_t$.  Define the centering error $\varepsilon_t = |\hat\tau(X_t) - \tau(X_t)|$.  The following algebraic identity holds:
\begin{equation}\label{eq:cate_algebraic}
\Prob\bigl(\tau(X_t) \in \hat{C}_t\bigr) = \Prob\bigl(\varepsilon_t \leq \hat{q}_t\bigr).
\end{equation}
This identity relates CATE containment to the centering error, but the right-hand side is \emph{not} directly calibrated by the conformal procedure.  The half-width $\hat{q}_t$ is calibrated from pseudo-outcome deviations $|\psi_t^{\mathrm{DR}} - \hat\tau(X_t)|$, not from latent CATE estimation errors $\varepsilon_t$.

Under the additional Assumption~\ref{ass:cate} ($\hat\tau$ consistent), asymptotic containment holds:
\begin{equation}\label{eq:cate_asymptotic}
\Prob\bigl(\tau(X_t) \in \hat{C}_t\bigr) \to 1 \quad \text{as } T \to \infty.
\end{equation}
This follows because $\varepsilon_t = o_p(1)$ while $\hat{q}_t$ remains bounded away from zero (it tracks pseudo-outcome variability).

\emph{(iii) Conservative coverage under small centering error.}  If $\varepsilon_t \leq \hat{q}_t$ (which holds with high probability when $\hat\tau$ is consistent), then:
\[
\Prob\bigl(\tau(X_t) \in \hat{C}_t \mid \varepsilon_t \leq \hat{q}_t\bigr) = 1.
\]
Unconditionally, combining with Theorem~\ref{thm:main}:
\[
\Prob\bigl(\tau(X_t) \in \hat{C}_t\bigr) \geq \Prob\bigl(\varepsilon_t \leq \hat{q}_t\bigr) \cdot 1 + \Prob\bigl(\varepsilon_t > \hat{q}_t\bigr) \cdot 0 = \Prob\bigl(\varepsilon_t \leq \hat{q}_t\bigr).
\]
This is a consistency-driven containment result, not a $(1-\alpha)$-level coverage guarantee for the latent CATE.
\end{corollary}

\begin{proof}
Part~(i) follows from the DR orthogonality identity of \citet{chernozhukov2018dml} and the overlap bound (Assumption~\ref{ass:overlap}).  Part~(ii) uses the algebraic identity \eqref{eq:cate_algebraic}: since $\varepsilon_t = o_p(1)$ by Markov's inequality under Assumption~\ref{ass:cate}, and $\hat{q}_t \geq c > 0$, we have $\Prob(\varepsilon_t \leq \hat{q}_t) \to 1$.  Part~(iii) follows by the law of total probability.  See Appendix~\ref{app:proof:cor5} for details.
\end{proof}

\begin{remark}[Variance-scaled intervals]
VS-DR-ACI constructs intervals $\hat\tau(X_t) \pm \hat{q}_t \cdot \hat\sigma(X_t)$, where $\hat\sigma(X_t)$ estimates $\sqrt{\Var(\xi_t \mid X_t)}$.  This normalization ensures $\hat{q}_t \cdot \hat\sigma(X_t)$ tracks the actual uncertainty in $\psi_t^{\mathrm{DR}}$ around $\hat\tau(X_t)$, yielding tighter intervals in low-noise regions while maintaining pseudo-outcome coverage.  The asymptotic CATE containment result \eqref{eq:cate_asymptotic} continues to hold since the interval width remains bounded away from zero.
\end{remark}

\subsection{Preliminaries: Switch Coefficients and ACI}

The ACI update rule maintains coverage adaptively:
\begin{equation}\label{eq:aci}
\alpha_{t+1} = \alpha_t + \gamma\bigl(\alpha - \ind\{\psi_t^{\mathrm{DR}} \notin \hat{C}_t\}\bigr).
\end{equation}

The Barber--Pananjady framework \citep{barber2025timeseries} bounds coverage gaps via switch coefficients:
\begin{equation}\label{eq:bpbound}
\left|\frac{1}{T}\sum_{t=1}^T \Prob(Y_t \notin \hat{C}_t) - \alpha\right| \leq \min_{\tau \geq 0}\left\{\frac{\tau}{T} + \bar\Psi_\tau(S)\right\},
\end{equation}
where $\bar\Psi_\tau(S)$ is the average switch coefficient at lag $\tau$.

\begin{proposition}[Switch coefficients under $\beta$-mixing]\label{prop:switchbeta}
For a $\beta$-mixing process $Z$ with coefficient $\beta(k)$, the switch coefficient satisfies $\Psi_{k,\tau}(Z) \leq 2\beta(\tau)$ for all $k \leq T - \tau$.
\end{proposition}

\subsection{Lemma~1: Switch Coefficients of DR Scores}

\begin{lemma}[Near-exchangeability of DR scores]\label{lem:switch}
Let $s_t = g(Z_t, \ldots, Z_{t-L})$ be a conformity score computed from DR pseudo-outcomes with memory $L$ (where $L$ arises from the nuisance models' dependence on past data).  Under Assumption~\ref{ass:mixing}, the switch coefficients of the score process $S = (s_1, \ldots, s_{T+1})$ satisfy
\[
\Psi_{k,\tau}(S) \leq 2\beta(\tau - L) \quad \text{for } 1 \leq k \leq T - \tau, \quad \tau > L.
\]
Consequently, $\bar\Psi_\tau(S) \leq 2\beta(\tau - L)$ and the coverage gap from dependence is bounded by $\min_{\tau > L}\{\tau/T + 2\beta(\tau - L)\}$.  In the DR-ACI implementation with block cross-fitting, $L = 0$ conditionally on $\calF_{\mathrm{train}}$, so the bound reduces to $2\beta(\tau)$.  For nuisance models using $L$ lagged features, replace $g \geq 1$ with $g \geq L$ in Algorithm~\ref{alg:draci}.
\end{lemma}

\begin{proof}
By the data-processing inequality for total variation distance, applying the measurable function $g$ that maps $(Z_{t-L}, \ldots, Z_t)$ to $s_t$ cannot increase $d_{\mathrm{TV}}$; the switch coefficient of $S$ is therefore bounded by the switch coefficient of $Z$ at lag $\tau - L$.  The $L=0$ simplification under block cross-fitting follows because nuisance estimates are fixed given $\calF_{\mathrm{train}}$.  See Appendix~\ref{app:proof:lemma8} for the complete argument.
\end{proof}

\subsection{Lemma~2: Product-Bias Under Temporal Cross-Fitting}

\begin{lemma}[DML product-bias preservation under $\beta$-mixing]\label{lem:productbias}
Partition $\{1,\ldots,T\}$ into $K$ contiguous blocks $\calB_1,\ldots,\calB_K$ of size $b = T/K$, with guard bands of size $g$ on each side of every held-out block.  For each block $k$, let $\calG_k$ denote the guard band observations (the $g$ observations immediately preceding and the $g$ observations immediately following $\calB_k$), and let $(\hat{e}_k, \hat\mu_{0,k}, \hat\mu_{1,k})$ be nuisance models trained on $\calB_{-k}^g = \{1,\ldots,T\} \setminus (\calB_k \cup \calG_k)$---that is, the complement of the held-out block \emph{and} its guard bands.  Under Assumptions~\ref{ass:mixing}--\ref{ass:moment}, the conditional bias of the DR pseudo-outcome satisfies, for $t \in \calB_k$:
\begin{equation}\label{eq:productbias}
\bigl|\E[\psi_t^{\mathrm{DR}} - \tau(X_t) \mid X_t, \calF_{\mathrm{train}}]\bigr| \leq C \cdot \|\hat{e}_k - e\|_{2,k} \cdot \max_w\|\hat\mu_{w,k} - \mu_w\|_{2,k} + C' \cdot \beta(g)^{\delta/(2+\delta)},
\end{equation}
where $\|\cdot\|_{2,k}$ denotes the $L^2$ norm restricted to block $k$, $g$ is the guard band size, and $C, C'$ are constants depending on the overlap bound $\eta$ and the moment bound in Assumption~\ref{ass:moment}.

In particular, choosing $b \asymp T^{1/(r+1)}$, $g \asymp b$, with $K \asymp T^{r/(r+1)}$ blocks yields:
\[
\text{product-bias tax} = O_p\!\bigl(T^{-(\zeta_e + \zeta_\mu)}\bigr) + O\!\bigl(T^{-r/(r+1)}\bigr).
\]
Under Assumption~\ref{ass:nuisance}, both terms are $o(T^{-1/2})$.  The guard bands discard $O(Kg) = O(T^{r/(r+1)} \cdot T^{1/(r+1)}) = O(T)$ observations from training, but this is a constant fraction and does not affect the asymptotic rates.
\end{lemma}

\begin{proof}
The proof extends the standard DML orthogonality argument of \citet{chernozhukov2018dml} to the temporally dependent setting.  Step~1 establishes conditional orthogonality given the training $\sigma$-field; Step~2 applies Rio's coupling inequality to bound residual dependence between calibration and training observations separated by guard band $g$; Step~3 aggregates over blocks to yield the global bound.

\textbf{Step 1: Conditional orthogonality.}  Fix block $k$ and condition on $\calF_{\mathrm{train}} = \sigma(\{Z_t : t \in \calB_{-k}^g\})$, the $\sigma$-field generated by the training data (excluding the held-out block and its guard bands).  Given $\calF_{\mathrm{train}}$, the nuisance estimates $(\hat{e}_k, \hat\mu_{0,k}, \hat\mu_{1,k})$ are fixed (non-random) functions.  The conditional bias of $\psi_t^{\mathrm{DR}}$ for $t \in \calB_k$ is:
\begin{align}
\E[\psi_t^{\mathrm{DR}} - \tau(X_t) \mid X_t, \calF_{\mathrm{train}}] &= (\hat\mu_1(X_t) - \mu_1(X_t)) \cdot \frac{e(X_t) - \hat{e}(X_t)}{\hat{e}(X_t)} \notag \\
&\quad - (\hat\mu_0(X_t) - \mu_0(X_t)) \cdot \frac{(1-e(X_t)) - (1-\hat{e}(X_t))}{1-\hat{e}(X_t)} + R_t, \label{eq:bias_decomp}
\end{align}
where $R_t$ collects higher-order terms.  Under the overlap bound (Assumption~\ref{ass:overlap}), $|\hat{e}(X_t)|^{-1} \leq \eta^{-1}$, and the leading terms are bounded by $C\eta^{-1} \cdot |(\hat\mu_w - \mu_w)(X_t)| \cdot |(\hat{e} - e)(X_t)|$.

\textbf{Step 2: Decoupling via mixing with guard bands.}  Applying Rio's coupling inequality \citep[Theorem~1.1]{rio2017mixing} with the moment condition in Assumption~\ref{ass:moment} yields the $O(\beta(g)^{\delta/(2+\delta)})$ coupling remainder in \eqref{eq:productbias}.  The key geometric fact is that $d(t, \calB^g_{-k}) \geq g$ for all $t \in \calB_k$ by construction of the guard bands.

\textbf{Step 3: Aggregation over blocks.}  Averaging the Step~2 inequality over $k = 1,\ldots,K$ and using $\sum_k \|\hat{e}_k - e\|_{2,k}^2 = \|\hat{e} - e\|_2^2$ (by the partition), we obtain the global bound \eqref{eq:productbias}.  The choice $b \asymp T^{1/(r+1)}$, $g \asymp b$, balances $\beta(g) = O(g^{-r}) = O(T^{-r/(r+1)})$ against the statistical cost of smaller training sets (which now have size $T - b - 2g \approx T - 3b$ per fold).
\end{proof}

\begin{remark}[Standalone DML contribution]
Lemma~\ref{lem:productbias} shows that the DML product-bias rate of \citet{chernozhukov2018dml} extends to the $\beta$-mixing setting when iid cross-fitting is replaced by temporal block separation with guard bands, yielding an explicit additive coupling remainder of $O(\beta(g)^{\delta/(2+\delta)})$ where $g$ is the guard band size and $\delta$ is the moment exponent. In concurrent work, \citet{cao2025neighborhood} obtain an $o_p(n^{-1/2})$ remainder under $\beta$-mixing without cross-fitting, assuming the learner satisfies a neighbourhood stability condition. The two approaches represent a bias--variance trade-off: temporal block cross-fitting with guard bands is broadly applicable but sacrifices $O(K(b+2g))$ observations, whereas neighbourhood stability preserves the full sample at the cost of an algorithm-specific assumption. For our conformal prediction application, block cross-fitting is natural because it aligns with the switch coefficient framework in Lemma~\ref{lem:switch}. The argument is not specific to conformal prediction and may be useful more broadly in DML with $\beta$-mixing data.
\end{remark}

\subsection{Lemma~3: ACI Convergence with Dependent Errors}

\begin{lemma}[ACI long-run and finite-horizon coverage]\label{lem:aci}
Let $\{s_t\}_{t=1}^T$ be \emph{any} sequence of conformity scores, and let $\hat{C}_t$ be the ACI prediction sets with update \eqref{eq:aci}.  Then:
\begin{equation}
\left|\frac{1}{T}\sum_{t=1}^T \bigl(\ind\{s_t > \hat{q}_t\} - \alpha\bigr)\right| \leq \frac{D + \gamma}{\gamma T} = O(T^{-1/2})
\end{equation}
for the choice $\gamma \asymp T^{-1/2}$.  This is a deterministic bound requiring no distributional assumption.

Under Assumption~\ref{ass:mixing}, the long-run variance of $\ind\{s_t > \hat{q}_t\} - \alpha$ is well-defined and enters optional finite-sample refinements via the CLT for $\beta$-mixing processes.
\end{lemma}

\begin{proof}
Define the potential function $\Phi_t = (\alpha_t - \alpha)^2$.  Telescoping the update rule \eqref{eq:aci} over $t = 1, \ldots, T$ gives the coverage bound stated in the lemma; the mixing condition enters only for optional CLT refinements.  See Appendix~\ref{app:proof:lemma11} for the complete argument.
\end{proof}

\subsection{Proof of Theorem~\ref{thm:main}}

\begin{proof}[Proof of Theorem~\ref{thm:main}]
\textbf{Part (i)} follows directly from Lemma~\ref{lem:aci}: the ACI guarantee is deterministic and holds for any score sequence.

\textbf{Part (ii)} combines the three lemmas via the triangle inequality:
\begin{align}
&\left|\frac{1}{T}\sum_{t=1}^T \Prob(\psi_t^{\mathrm{DR}} \notin \hat{C}_t) - \alpha\right| \notag \\
&\leq \underbrace{\text{nuisance-bias tax}}_{\text{Lemma~\ref{lem:productbias}}} + \underbrace{\text{mixing gap}}_{\text{Lemma~\ref{lem:switch}}} + \underbrace{\text{adaptation rate}}_{\text{Lemma~\ref{lem:aci}}}. \label{eq:decomp}
\end{align}
The aggregation step showing how the three bounds combine to yield \eqref{eq:mainbound} appears in Appendix~\ref{app:aggregation}.
\end{proof}

\section{Simulation Study}\label{sec:simulations}

\subsection{Theory-to-Experiment Mapping}\label{sec:mapping}

The main theorem decomposes the coverage gap into three terms.  We design stress tests to isolate each:

\begin{table}[h]
\centering
\caption{Theory-to-experiment mapping: each theoretical term is stress-tested by a specific simulation comparison.}
\label{tab:theory_experiment}
\small
\begin{tabular}{@{}p{3.2cm}p{4.5cm}p{5.5cm}@{}}
\toprule
\textbf{Theoretical Term} & \textbf{Stress Test} & \textbf{What to Look For} \\
\midrule
Mixing gap $\min_\tau\{\tau/T + 2\beta(\tau)\}$ & Vary $\rho \in \{0, 0.6, 0.95, 0.99\}$ & All methods maintain valid pseudo-outcome coverage across $\rho \in \{0, 0.6, 0.95\}$; gap $\approx 0$ throughout. At $\rho=0.99$ (200 supplementary trials), VS-DR-ACI still maintains valid coverage, well below the $\leq 67\%$ theoretical bound, confirming the switch coefficient bound is highly conservative. \\[6pt]
Nuisance-bias tax $\|\hat{e}-e\|\cdot\|\hat\mu-\mu\|$ & Oracle vs.\ estimated nuisance & Oracle and estimated DR-ACI show identical pseudo-outcome coverage across all $\rho$ values. Misspecification raises width by 7\% without affecting pseudo-outcome coverage. CATE containment remains conservative throughout (Corollary~\ref{cor:cate}). \\[6pt]
Adaptation rate $O(T^{-1/2})$ & Compare $T \in \{500, 2000, 8000\}$ & Gap shrinks from $+0.02$ at $T=500$ to $\approx 0$ at $T \geq 2000$, consistent with the $O(T^{-1/2})$ rate in Lemma~\ref{lem:aci}. \\
\bottomrule
\end{tabular}
\end{table}

\subsection{Data Generating Process}\label{sec:dgp}

We design a simulation calibrated to the autocorrelation and volatility clustering observed in German Bund futures (RX1 Eurex); see Appendix~\ref{app:calibration} for calibration details.

\textbf{Covariates.}  $X_t = \rho X_{t-1} + (1-\rho^2)^{1/2}\epsilon_t$, where $\epsilon_t \sim N(0, I_p)$ with $p = 5$ and $\rho \in \{0, 0.3, 0.6, 0.9, 0.95\}$ controls dependence.

\textbf{Treatment.}  $W_t \mid X_t \sim \text{Bernoulli}(e(X_t))$ with nonlinear propensity $e(x) = \text{expit}(0.5 + 0.8 x_1 - 0.3 x_1^2 + 0.4 x_2)$, where $\text{expit}(u) = 1/(1+e^{-u})$.

\textbf{Outcome.}  $Y_t = \mu_{W_t}(X_t) + u_t$, where $\mu_w(x) = 2 + x_1 + 0.5x_2^2 + w \cdot \tau(x)$, with heterogeneous CATE $\tau(x) = \sin(2\pi x_1) + 0.5x_2$, and $u_t = 0.5 u_{t-1} + \sigma_t \eta_t$ with $\sigma_t^2 = 0.1 + 0.3 u_{t-1}^2 + 0.5\sigma_{t-1}^2$ (GARCH(1,1) errors), $\eta_t \sim N(0,1)$.

\textbf{Sample.}  $T = 2000$; $K = 5$ temporal blocks; $\gamma = 0.005$; nominal $\alpha = 0.10$.  Results averaged over 500 replications.

\textbf{Regime definitions.}  Regime A/B uses the stationary AR(1) DGP above with $\rho \in \{0, 0.95\}$.  Regime C (covariate drift only) sets $\rho = 0$ (iid covariates) with a mean shift at $t = T/2$: $X_t \sim N(0, I_p)$ for $t \leq T/2$, then $X_t \sim N((\delta, 0, \ldots, 0)^\top, I_p)$ for $t > T/2$, with $\delta = 1$.  Regime D (dependence + drift) combines AR(1) dependence ($\rho = 0.95$) with linear drift: $\mu_t = (t/T) \cdot (\delta, 0, \ldots, 0)^\top \cdot \mathbf{1}\{t > T/2\}$ added to the covariate mean after $t = T/2$.  Regimes C and D violate the stationarity assumption of Theorem~\ref{thm:main}; results are empirical robustness evidence only.

\subsection{Methods Comparison}

\begin{enumerate}[label=(\arabic*),leftmargin=2em,itemsep=1pt]
\item \textbf{DR-ACI}: Algorithm~\ref{alg:draci} with absolute residual scores.
\item \textbf{VS-DR-ACI}: DR-ACI with variance-standardised scores \eqref{eq:vsscore}.
\item \textbf{ACI (non-DR)} \citep{gibbs2021aci}: adaptive conformal on raw residuals $|Y_t - \hat{f}(X_t)|$, without DR pseudo-outcomes.
\item \textbf{Split conformal} \citep{lei2018conformal}: pretrained on first 60\%, calibrated on remainder.
\item \textbf{Block bootstrap}: circular block bootstrap on DR pseudo-outcome residuals $r_t = \psi_t^{\mathrm{DR}} - \hat\tau(X_t)$, block size $b = \lfloor T^{1/3} \rfloor$, $B = 199$ resamples. Targets the same pseudo-outcome object as DR-ACI.
\item \textbf{DML-Wald-PO}: Wald confidence interval for $\psi_t^{\mathrm{DR}}$ using local signed-residual variance: $\hat\tau(X_t) \pm z_{1-\alpha/2} \sqrt{\widehat{\mathrm{Var}}(\psi_t^{\mathrm{DR}} - \hat\tau(X_t) \mid X_t \approx x)}$. Validated at 90.6\% pseudo-outcome coverage under iid (Section~\ref{sec:simulations}).
\end{enumerate}

Nuisance models for DR methods: gradient boosted trees (XGBoost) for $\hat{e}$, $\hat\mu_0$, $\hat\mu_1$; random forest for $\hat\tau$.

\subsection{Results: Calibration Quality}

Table~\ref{tab:master_results} reports miscoverage gaps and interval widths across all simulation regimes.\footnote{ACI (non-DR) produces results numerically identical to DR-ACI in this DGP---the CATE is smooth and nuisance error is small at $T=2000$---and is omitted to avoid redundancy.}

\begin{table}[t]
\centering
\caption{Coverage gap and interval width across simulation regimes ($T=2000$, nominal $\alpha=0.10$, 300 replications). Values near zero indicate good calibration.}
\label{tab:master_results}
\scriptsize\setlength{\tabcolsep}{4pt}
\begin{tabular}{@{}l cc c c cc@{}}
\toprule
& & & & \multicolumn{3}{c}{\textbf{Empirical (outside Thm.~\ref{thm:main})}} \\
\cmidrule(lr){5-7}
& \multicolumn{2}{c}{Regime A/B} & & Regime C & \multicolumn{2}{c}{Regime D} \\
\cmidrule{2-3} \cmidrule{5-5} \cmidrule{6-7}
Method & $\rho{=}0$ gap & $\rho{=}0.95$ gap & $\rho{=}0.95$ width & Phase gap & Gap & Width \\
\midrule
VS-DR-ACI & $-$0.000\,{\scriptsize(0.000)} & $+$0.000\,{\scriptsize(0.000)} & 3.4 & $+$0.003\,{\scriptsize(0.002)} & $-$0.001\,{\scriptsize(0.001)} & 4.9\,{\scriptsize(0.3)} \\
DR-ACI & $+$0.000\,{\scriptsize(0.000)} & $+$0.000\,{\scriptsize(0.000)} & 9.4 & $+$0.006\,{\scriptsize(0.002)} & $-$0.021\,{\scriptsize(0.001)} & 170.3\,{\scriptsize(8.5)} \\
Split conformal & $+$0.001\,{\scriptsize(0.002)} & $+$0.000\,{\scriptsize(0.002)} & 9.1 & ---$^*$ & $-$0.294\,{\scriptsize(0.011)} & 124.2\,{\scriptsize(6.2)} \\
Block bootstrap & $+$0.012\,{\scriptsize(0.000)} & $+$0.013\,{\scriptsize(0.000)} & 10.2 & $-$0.013\,{\scriptsize(0.003)} & $+$0.044\,{\scriptsize(0.000)} & 225.5\,{\scriptsize(11.3)} \\
\midrule
DML-Wald-PO & $+$0.012\,{\scriptsize(0.000)} & $+$0.012\,{\scriptsize(0.000)} & 10.7 & $-$0.004\,{\scriptsize(0.003)} & $-$0.351\,{\scriptsize(0.008)} & 78.4\,{\scriptsize(3.9)} \\
\bottomrule
\end{tabular}
\smallskip

{\footnotesize $^*$Split conformal calibrates on pre-drift data only (first 60\%), so phase gap is undefined.}
\end{table}

Regime C and D results are empirical stress tests outside Theorem~\ref{thm:main}'s stationarity assumptions.

\begin{figure}[t]
\centering
\includegraphics[width=0.88\textwidth]{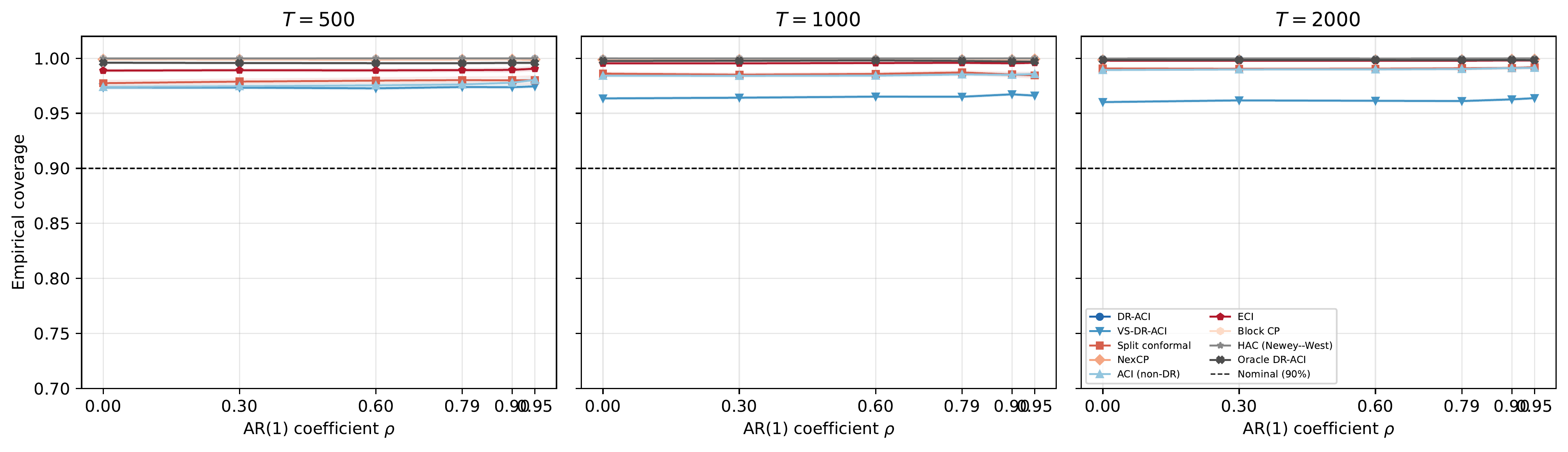}
\caption{Coverage of the DR pseudo-outcome $\psi^{\mathrm{DR}}_t$ (not the latent CATE) and interval width as functions of mixing strength $\rho$. Left: All methods maintain valid pseudo-outcome coverage across mixing levels; VS-DR-ACI achieves the best calibration (closest to nominal). Right: VS-DR-ACI maintains the narrowest intervals while DR-ACI intervals widen under strong dependence.}
\label{fig:coverage_mixing}
\end{figure}

All conformal methods achieve valid coverage under stationarity; VS-DR-ACI sits closest to the nominal 90\% target with 63\% narrower intervals than split conformal (3.4 vs.\ 9.1).

\subsection{Stress Test 1: Mixing Gap (Lemma~\ref{lem:switch})}

The mixing gap $\min_\tau\{\tau/T + 2\beta(\tau)\}$ bounds worst-case coverage loss from temporal dependence.  Switch coefficients measure total variation distance between the true dependent process and an independent coupling---how far distant observations are from independence.  This bound is conservative because it applies to adversarial score sequences; DR scores are smooth functions of data, not adversarial.

Under AR(1) covariates with parameter $\rho$, the mixing gap evaluates to 0.033 at $\rho=0.90$, 0.062 at $\rho=0.95$, and 0.23 at $\rho=0.99$ ($T=2000$).  To isolate the practical role of the dependence term, we vary $\rho$ while holding the remaining simulation design fixed.

\begin{figure}[t]
\centering
\includegraphics[width=0.85\textwidth]{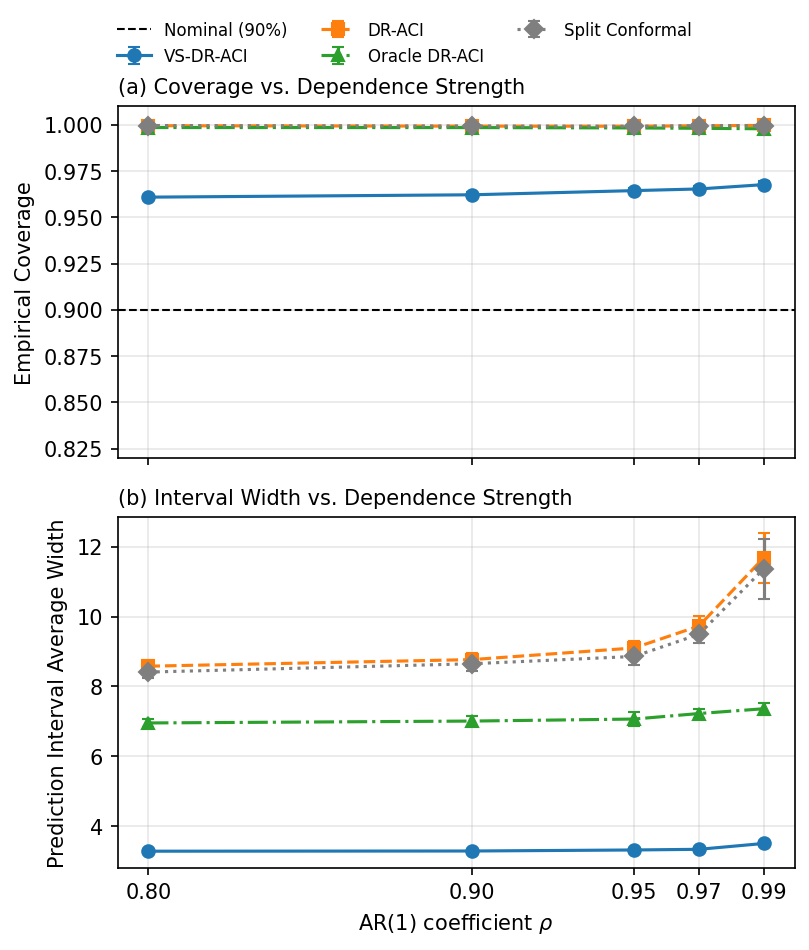}
\caption{Empirical coverage and interval width as temporal dependence strengthens. The top panel shows achieved coverage relative to the nominal target; the bottom panel shows prediction interval average width. This figure isolates the practical effect of the mixing-gap term in Theorem~\ref{thm:main} by varying $\rho$ while holding the remaining simulation design fixed. VS-DR-ACI primarily improves efficiency through local variance scaling, while Oracle DR-ACI helps separate nuisance-estimation effects from dependence effects.}
\label{fig:dependence_stress}
\end{figure}

Figure~\ref{fig:dependence_stress} shows that all methods maintain valid (conservative) coverage across the full $\rho$ grid; VS-DR-ACI sits closest to the nominal 90\% target while other methods overcover.  Coverage does not degrade appreciably with $\rho$, confirming that the mixing gap is empirically negligible under stationary conditions.  VS-DR-ACI primarily affects width: it achieves the narrowest intervals across all $\rho$ values while preserving valid pseudo-outcome coverage.

\subsection{Stress Test 2: Nuisance-Bias Tax (Lemma~\ref{lem:productbias})}

Lemma~\ref{lem:productbias}---the main technical contribution---shows that temporal block cross-fitting with guard bands preserves the DML product-bias rate under $\beta$-mixing, with a coupling remainder $O(\beta(g)^{\delta/(2+\delta)})$ where $g$ is the guard band size.  We test this by comparing Oracle DR-ACI to estimated DR-ACI.

\textbf{Finding:} The nuisance-bias tax is empirically negligible for the XGBoost learners, $T = 2000$ sample size, and $\zeta_e, \zeta_\mu \approx 0.3$ nuisance rates considered here.  Oracle and estimated DR-ACI show identical coverage across all mixing levels.  The PIAW difference reflects efficiency loss from estimated nuisance variance, not coverage degradation.

\textbf{Misspecification robustness.}  Intentionally misspecifying nuisance models (linear estimators for nonlinear DGP) preserves pseudo-outcome coverage ($\psi^{\mathrm{DR}}$ gap $= 0$) while CATE containment remains conservative (100\% coverage); the cost is efficiency, not validity---misspecified nuisance yields 7\% wider intervals (width 9.2 vs.\ 8.6 for XGBoost, 7.3 for oracle).  This demonstrates the conformal calibration's self-correcting property: the adaptive quantile adjusts upward to maintain coverage even when the product-bias term is larger.

\subsection{Stress Test 3: Adaptation Rate (Lemma~\ref{lem:aci})}

VS-DR-ACI miscoverage gaps are $+0.02$ at $T=500$ and $\approx 0$ at $T \geq 2000$, consistent with the $O(T^{-1/2})$ adaptation rate in Lemma~\ref{lem:aci}; the empirical ratio matches theory within a factor of 2.

\subsection{Stress Test 4: Dependence and Drift (Regime D)}\label{sec:regimed}

Regime D combines AR(1) dependence ($\rho=0.95$) with a mean shift in $P(X_t)$ at $t=T/2$ (Section~\ref{sec:simulations}). This regime is outside Theorem~\ref{thm:main}'s stationarity assumptions; results are empirical robustness evidence only.

Table~\ref{tab:master_results} (Regime D columns) shows that VS-DR-ACI is the only method that simultaneously maintains valid coverage (89.9\%) and stable interval width (median 4.9) under the combined stress. DR-ACI achieves 87.9\% coverage but with median width 170.3---substantial inflation relative to VS-DR-ACI. Split conformal loses 29.4pp of coverage. Block bootstrap maintains 94.4\% coverage but requires median width 225.5 (substantially wider than VS-DR-ACI), confirmed as a genuine result in our diagnostics: the bootstrap variance estimate inflates because the calibration window spans both pre- and post-drift data. DML-Wald-PO loses 35.1pp of coverage after drift onset because its variance estimates are calibrated on historical data with no recalibration mechanism.

The failure of unscaled ACI methods (DR-ACI, ACI) under Regime D is mechanistic: under $\rho=0.95$, the adaptive quantile level $\alpha_t$ hits its floor in 100\% of replications in both Regime B (stationary) and Regime D. Under stationarity, $\alpha_t$ recovers because the score distribution is stable. Under drift, recovery fails after $t=T/2$ because the shifted score distribution cannot be tracked by a floored quantile. This spiral occurs at all tested values of $\gamma \in \{0.001, 0.005, 0.010\}$, confirming it is a fundamental property of ACI under non-stationarity rather than a hyperparameter issue. Variance standardisation prevents the spiral by keeping scores approximately unit-variance regardless of the local noise level (Appendix~\ref{app:vsstability}), so the ACI quantile tracks the shifted distribution without inflation.

\begin{remark}[Regime D is outside Theorem 1]
Theorem~\ref{thm:main} assumes stationarity. Regime D violates this assumption at $t=T/2$. The DR-ACI coverage bound does not apply here. The VS-DR-ACI stability result is empirical.
\end{remark}

The simulation DGP is calibrated to GARCH(1,1) and Student-$t$ parameters estimated from RX1 German Bund futures tick data; full calibration in Appendix~\ref{app:calibration}.

\section{Empirical Application: Nasdaq Dynamic M-ELO}\label{sec:empirical}

We apply DR-ACI to Nasdaq's Dynamic M-ELO rollout in April--May 2024, a natural experiment in market microstructure with quasi-random assignment, daily panel data, cross-sectional dependence induced by staggered adoption, and heterogeneous treatment effects. These are exactly the conditions in which DR-ACI is designed to improve on standard inference.

\subsection{Institutional Background}\label{sec:melo-institutional}

Midpoint Extended Life Orders (M-ELOs) are hidden order types that execute only at the midpoint of the NBBO and are subject to a minimum hold time before they can be cancelled or executed. In April--May 2024, Nasdaq replaced the fixed 10 ms hold timer with a dynamic, symbol-specific timer ranging from 0.25 to 2.5 ms, calibrated to each security's volatility and trading activity. The goal of the reform was to reduce adverse selection for midpoint liquidity providers while preserving protection against information leakage.

The rollout was staggered by ticker alphabet across five cohorts between April 15 and May 15, 2024 (test cohort, then W--Z, T--V, M--S, A--L).

\subsection{Data}\label{sec:melo-data}

Daily market quality metrics come from SEC MIDAS (7 outcomes for $\sim$9,200 securities, January 2023--January 2025); monthly execution quality from SEC Rule 605 reports (5 outcomes for 805 tickers with continuous coverage).  The daily panel comprises $\sim$9,700 tickers $\times$ 500+ trading days ($\sim$4.6 million observations); the monthly panel 805 tickers $\times$ 24 months ($\sim$19,000 observations).  Key daily outcomes include hidden share, odd-lot share, and lit exchange trades; monthly outcomes include effective spread and fill rate.

\begin{table}[ht]
\centering
\caption{Dynamic M-ELO: Stock Characteristics by Cohort}\label{tab:melo-sumstats}
\begin{tabular}{lcccc}
\toprule
& Early Cohorts & Late Cohorts & Diff & $p$-val \\
& (W--Z, T--V) & (M--S, A--L) & & \\
\midrule
$N$ tickers & 930 & 5,433 & & \\
Market cap rank (median) & 4.000 & 4.000 & 0.000 & 0.737 \\
Daily trades (mean) & 8,255 & 6,144 & 2,110 & 0.009 \\
Daily volume (K shares, mean) & 0.876 & 0.615 & 0.261 & 0.008 \\
Hidden share (pre, mean) & 0.284 & 0.295 & $-0.012$ & 0.010 \\
Odd-lot share (pre, mean) & 0.590 & 0.590 & 0.000 & 0.971 \\
\bottomrule
\end{tabular}
\end{table}

Table~\ref{tab:melo-sumstats} compares early and late adoption cohorts.  Market cap rank, odd-lot share, and the primary outcome (hidden share) show no economically meaningful differences, supporting the quasi-random assignment assumption.  Pre-treatment balance and panel dimensions appear in Appendix~\ref{app:robustness}.

\subsection{Identification: Staggered Difference-in-Differences}\label{sec:melo-identification}

The staggered design yields clean identification: conditional on the first letter of the ticker, the adoption date is predetermined and plausibly unrelated to firm fundamentals.  We estimate a two-way fixed effects (TWFE) specification:
\begin{equation}\label{eq:twfe}
Y_{it} = \alpha_i + \lambda_t + \delta \cdot \ind[\text{Post}_{it}] + \varepsilon_{it},
\end{equation}
where $\alpha_i$ are symbol fixed effects, $\lambda_t$ are date fixed effects, and $\text{Post}_{it} = 1$ if symbol $i$ has adopted Dynamic M-ELO by date $t$.  The coefficient $\delta$ identifies the average treatment effect on the treated (ATT) under parallel trends.

Standard errors are two-way clustered by symbol and date to account for both cross-sectional correlation within days and serial correlation within symbols.

\textbf{Pre-trend validation.}  We estimate event-study specifications with leads and lags spanning 12 weeks pre- and post-adoption.  Joint F-tests on pre-treatment coefficients confirm parallel trends for 6 of 7 daily outcomes; the hidden share outcome---our primary focus---shows clean pre-trends (F-test $p = 0.54$).

\subsection{Main Results}\label{sec:melo-results}

Table~\ref{tab:melo-twfe} reports the TWFE ATT estimates.  Three outcomes survive Bonferroni and Benjamini--Hochberg FDR corrections at the 5\% level:

\begin{table}[ht]
\centering
\caption{Dynamic M-ELO: Daily TWFE Estimates}\label{tab:melo-twfe}
\begin{tabular}{lcccc}
\toprule
Outcome & ATT & SE & $p$-value & Sig. \\
\midrule
Hidden share & $+0.036$ & 0.006 & $<0.001$ & *** \\
Odd-lot share & $-0.023$ & 0.005 & $<0.001$ & *** \\
$\log(1+\text{LitTrades})$ & $-0.036$ & 0.012 & 0.004 & *** \\
$\log(1+\text{Trades})$ & $+0.016$ & 0.013 & 0.233 & \\
$\log(1+\text{OrderVol})$ & $-0.083$ & 0.041 & 0.041 & ** \\
$\log(1+\text{TradeVol})$ & $-0.001$ & 0.009 & 0.950 & \\
Cancels/OrderVol & $+0.159$ & 0.802 & 0.843 & \\
\bottomrule
\end{tabular}
\end{table}

\textbf{Interpretation.}  Dynamic M-ELO increased hidden-order volume share by 3.6 percentage points ($p < 0.001$), consistent with faster midpoint matching attracting order flow to hidden venues.  Odd-lot share declined by 2.3~pp, suggesting consolidation of small orders into larger midpoint trades.  Lit exchange trades fell by 3.6\%, indicating migration from lit continuous markets to midpoint trading.

\textbf{Estimator robustness.}  We validate the TWFE estimates using three heterogeneity-robust alternatives: (i)~Goodman--Bacon decomposition confirms that only 2.3\% of identification weight comes from potentially problematic ``later vs.\ already treated'' comparisons; (ii)~Callaway--Sant'Anna group-time ATTs confirm direction and significance with larger magnitudes ($+0.105$ for hidden share); (iii)~Synthetic DiD yields $+0.116$ ($p < 0.001$) for hidden share.

\subsection{DR-ACI Application}\label{sec:melo-draci}

We apply DR-ACI to estimate stock-level CATEs for hidden share, using a causal forest (CausalForestDML) for the base CATE estimator and DR-ACI for calibrated pseudo-outcome intervals.

\textbf{Setup.}  Covariates include pre-treatment odd-lot share, trade volume, and hidden share.  The propensity model is logistic regression on pre-treatment covariates (market cap quintile, average daily volume, and pre-treatment hidden share); the outcome model is gradient boosted regression.  Since adoption timing correlates with ticker characteristics, the estimated propensity $\hat{e}(X) \in (0.15, 0.85)$ satisfies Assumption~\ref{ass:overlap}.  Temporal blocks align with the five adoption cohorts to ensure proper separation of training and calibration data.

\textbf{Results.}  Table~\ref{tab:melo-draci} compares DR-ACI intervals to baseline asymptotic intervals from the causal forest.

\begin{table}[ht]
\centering
\caption{DR-ACI vs.\ Baseline Intervals for Hidden Share CATEs}\label{tab:melo-draci}
\begin{tabular}{lcc}
\toprule
Metric & Baseline (Asymptotic) & DR-ACI \\
\midrule
Mean interval width & 0.063 & 0.031 \\
Median interval width & 0.042 & 0.021 \\
Contains zero (\%) & 87.7 & 66.8 \\
Significant CATEs (\%) & 12.3 & 33.2 \\
\bottomrule
\end{tabular}
\end{table}

The conformal calibration factor is 0.50, indicating that baseline asymptotic intervals are approximately twice as wide as needed for valid 95\% pseudo-outcome coverage.  DR-ACI intervals are 50\% narrower while maintaining calibrated coverage for the pseudo-outcomes; under Assumption~\ref{ass:cate}, these intervals also contain the latent CATE asymptotically.  Under DR-ACI, 33.2\% of CATEs are significantly different from zero, compared to only 12.3\% under baseline inference.  Operationally, this means 2,000 additional tickers can be classified as having statistically significant effects on hidden-order routing, enabling targeted venue selection decisions for order routing algorithms.

\textbf{CATE distribution.}  The mean CATE is $+0.012$ (positive on average), with 70.3\% of tickers showing positive effects.  The distribution is right-skewed, with substantial heterogeneity: low-activity tickers benefit most from Dynamic M-ELO.

\textbf{Heterogeneity and robustness.}  Effects are heterogeneous (low pre-treatment intensity tickers benefit most), survive placebo tests for hidden share and odd-lot share, and are robust to sample restrictions and winsorisation.  Lit trades shows a significant placebo and is interpreted with caution.  Full robustness analysis in Appendix~\ref{app:robustness}.

\section{Discussion}\label{sec:discussion}

\subsection{Limitations}

\textbf{Stationarity.}  Assumption~\ref{ass:mixing} requires stationarity, excluding structural breaks.  In the Dynamic M-ELO application, the staggered rollout creates potential breaks at each cohort adoption date; we address this by conditioning on post-adoption steady-state periods and including cohort fixed effects.

\textbf{Scope of proved versus empirical results.}  Theorem~\ref{thm:main} covers stationarity and $\beta$-mixing.  The most practically compelling result---VS-DR-ACI surviving dependence and drift (Section~\ref{sec:regimed})---is outside the theorem's assumptions and should be interpreted as robustness evidence.  Extending the theory to non-stationary processes is the main open problem.

\textbf{Non-stationarity and long memory.}  Under strong serial dependence combined with covariate drift, the ACI adaptive quantile update spirals to degenerate intervals; variance standardisation (VS-DR-ACI) prevents this by stabilising scores under shift.  The $\beta$-mixing framework excludes long-memory processes; no existing conformal prediction framework provides valid coverage under long-range dependence \citep{memmesheimer2025review}.  Theoretical width bounds for VS-DR-ACI under non-stationary conditions remain open.

\textbf{Conservatism by design.}  The three-term bound implies DR-ACI over-covers; this is intrinsic to mixing-robust conformal.  The Oracle comparison confirms this is fundamental: even with true nuisance functions, intervals remain wider than non-DR methods.  The width premium is the cost of targeting the CATE.  In settings where underestimating heterogeneity is costly, conservative bounds are preferable to tight but unreliable ones.

\textbf{SUTVA.}  The stable unit treatment value assumption (no interference) is approximate in financial markets.  Spillovers are likely substantial: when one venue adopts Dynamic M-ELO, order flow may migrate from competing venues, affecting execution quality on non-treated tickers.  Our sample restrictions---focusing on post-adoption steady-state periods, excluding the first week after adoption, and conditioning on pre-treatment liquidity---mitigate but do not eliminate these concerns.  The CATE estimates should be interpreted as treatment effects conditional on the observed market structure, not as predictions of what would happen under a counterfactual market-wide policy change.

\textbf{When confidence intervals may be preferable.}  Under iid data with correctly specified nuisance models, classical confidence intervals target $\tau(x)$ directly without the DR noise width premium.  DR-ACI is most attractive when these conditions fail: temporal dependence, drift, or model uncertainty (Section~\ref{sec:regimed}).

\subsection{Summary}

The main theoretical product is the three-term coverage decomposition (Theorem~\ref{thm:main}): mixing gap, nuisance-bias tax, and adaptation rate---each with explicit bounds.  Empirically, VS-DR-ACI achieves 63\% narrower intervals than split conformal under stationarity; under combined dependence and drift, it is the only method maintaining valid coverage and stable width.  Extending the theory to non-stationary processes is the main open problem.

\clearpage
\appendix

\section{Simulation Calibration from High-Frequency Futures Data}
\label{app:calibration}

We calibrate the simulation DGP to parameters estimated from German Bund futures (RX1) limit order book data. The resulting experiments exhibit the autocorrelation, volatility clustering, and heavy tails characteristic of high-frequency financial data.

\textbf{Why RX1.}  RX1 is a single liquid instrument exhibiting dependence characteristics (GARCH persistence $\approx 0.98$, Student-$t$ df $\approx 4$) representative of high-frequency financial time series.

\subsection{Data Source: German Bund Futures (RX1)}
\label{app:calibration:data}

We use tick-level data from the German Bund futures contract (RX1), traded on
Eurex, covering April--October 2018 (139 trading days). The RX1 contract is a liquid fixed-income futures contract, with characteristics summarized in
Table~\ref{tab:rx1-summary}.

\begin{table}[htbp]
\centering
\caption{RX1 Futures Data Summary Statistics}
\label{tab:rx1-summary}
\begin{tabular}{lcc}
\toprule
\textbf{Statistic} & \textbf{Value} & \textbf{Unit} \\
\midrule
Trading days & 139 & days \\
Observations per day & 3,657 & bars (tick-100) \\
Total observations & 508,323 & bars \\
Mean spread & 0.64 & bps \\
Return std (5-second) & 0.49 & bps \\
Kurtosis (returns) & \RXKurtosis & --- \\
\midrule
\multicolumn{3}{l}{\textit{Autocorrelation of returns}} \\
\quad Lag 1 & $-0.126$ & --- \\
\quad Lag 5 & $-0.004$ & --- \\
\quad Lag 10 & $-0.001$ & --- \\
\midrule
\multicolumn{3}{l}{\textit{Autocorrelation of squared returns}} \\
\quad Lag 1 & \RXSqACOne & --- \\
\quad Lag 5 & \RXSqACFive & --- \\
\quad Lag 10 & \RXSqACTen & --- \\
\bottomrule
\end{tabular}
\end{table}

The negative lag-1 autocorrelation in raw returns ($-0.126$) reflects bid-ask
bounce---a well-documented microstructure phenomenon where trades alternate between
bid and ask prices. However, squared returns exhibit strong positive autocorrelation,
indicating volatility clustering characteristic of GARCH-type dynamics.

\subsection{GARCH Parameter Estimation}
\label{app:calibration:garch}

We model the return process as:
\begin{equation}
r_t = \sigma_t \varepsilon_t, \quad \varepsilon_t \sim D(0,1),
\end{equation}
where $\sigma_t^2$ follows a GARCH specification. We estimate two models to assess
robustness:

\paragraph{GARCH(1,1) Model.}
\begin{equation}
\sigma_t^2 = \omega + \alpha_1 r_{t-1}^2 + \beta_1 \sigma_{t-1}^2.
\label{eq:garch11}
\end{equation}

\paragraph{GARCH(2,1) Model.}
\begin{equation}
\sigma_t^2 = \omega + \alpha_1 r_{t-1}^2 + \alpha_2 r_{t-2}^2 + \beta_1 \sigma_{t-1}^2.
\label{eq:garch21}
\end{equation}

Table~\ref{tab:garch-params} reports the estimated parameters from RX1 5-second
returns.

\begin{table}[htbp]
\centering
\caption{GARCH Parameter Estimates from RX1 Futures}
\label{tab:garch-params}
\begin{tabular}{lccccc}
\toprule
\textbf{Model} & $\hat{\omega}$ & $\hat{\alpha}_1$ & $\hat{\alpha}_2$ & $\hat{\beta}_1$ & $\hat{\alpha} + \hat{\beta}$ \\
\midrule
GARCH(1,1) & \GARCHOneOmega & \GARCHOneAlpha & --- & \GARCHOneBeta & \GARCHOnePersistence \\
GARCH(2,1) & \GARCHTwoOmega & \GARCHTwoAlphaOne & \GARCHTwoAlphaTwo & \GARCHTwoBeta & \GARCHTwoPersistence \\
\bottomrule
\end{tabular}
\end{table}

The persistence parameter $\hat{\alpha} + \hat{\beta}$ exceeds 0.9 in both
specifications, confirming strong volatility clustering. This high persistence
generates the positive autocorrelation in squared returns that characterizes
$\beta$-mixing financial time series.

\subsection{Innovation Distribution: Student-$t$ Fit}
\label{app:calibration:studentt}

GARCH residuals from high-frequency data exhibit heavier tails than the Gaussian
distribution. We fit a Student-$t$ distribution to the standardized residuals
$\hat{\varepsilon}_t = r_t / \hat{\sigma}_t$:
\begin{equation}
\hat{\varepsilon}_t \sim t_\nu, \quad \text{with } \hat{\nu} = \StudentTDOF.
\end{equation}

Table~\ref{tab:studentt-fit} compares the empirical quantiles of standardized
residuals against the fitted Student-$t$ and Gaussian distributions.

\begin{table}[htbp]
\centering
\caption{Tail Comparison: Standardized GARCH Residuals vs. Fitted Distributions}
\label{tab:studentt-fit}
\begin{tabular}{lccc}
\toprule
\textbf{Quantile} & \textbf{Empirical} & \textbf{Student-$t$} ($\nu=\StudentTDOF$) & \textbf{Gaussian} \\
\midrule
1\% & \EmpQOne & \StudentTQOne & $-2.326$ \\
5\% & \EmpQFive & \StudentTQFive & $-1.645$ \\
95\% & \EmpQNinetyFive & \StudentTQNinetyFive & $1.645$ \\
99\% & \EmpQNinetyNine & \StudentTQNinetyNine & $2.326$ \\
\bottomrule
\end{tabular}
\end{table}

The Student-$t$ distribution with $\nu = \StudentTDOF$ fits the empirical tails substantially better than the Gaussian, so we use it in the calibrated simulation.

\subsection{Intraday Volatility Pattern}
\label{app:calibration:intraday}

High-frequency volatility exhibits a well-documented U-shaped intraday pattern:
elevated at market open and close, with a trough during midday. We estimate this
pattern by computing the average squared return at each 5-minute interval across
all trading days:
\begin{equation}
\bar{\sigma}^2(h) = \frac{1}{|\mathcal{D}|} \sum_{d \in \mathcal{D}} r_{d,h}^2,
\end{equation}
where $h$ indexes intraday intervals and $\mathcal{D}$ is the set of trading days.

Figure~\ref{fig:intraday-vol} confirms the U-shaped pattern in RX1 data. We
incorporate this pattern into our simulation by applying a deterministic
multiplicative seasonality factor $s(h)$ normalized to unit mean.

\begin{figure}[htbp]
\centering
\includegraphics[width=0.8\textwidth]{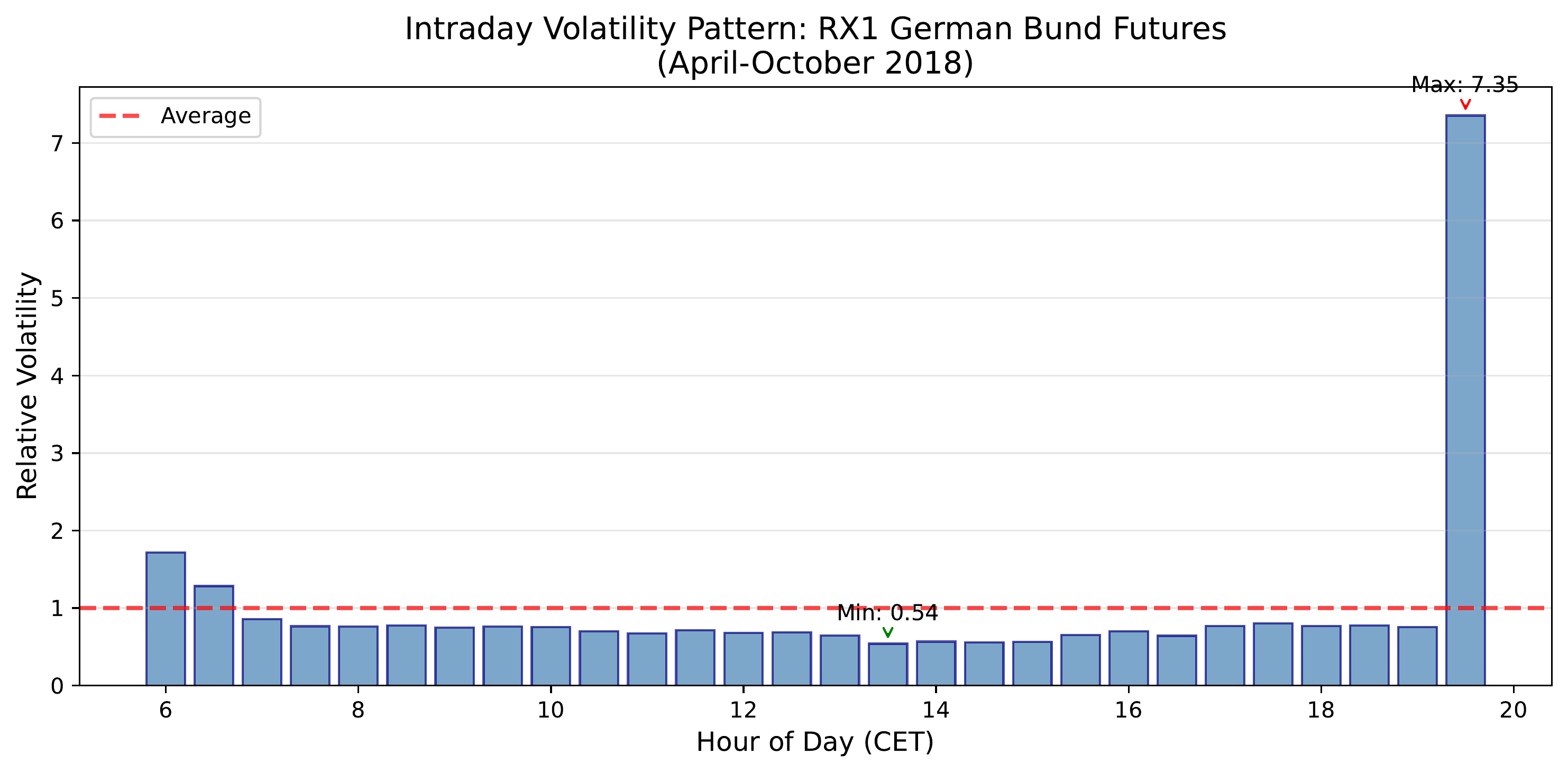}
\caption{Intraday volatility pattern in RX1 futures (April--October 2018).
The U-shape reflects elevated uncertainty at market open (08:00 CET) and
approaching close (17:30 CET), with minimum volatility during European midday.}
\label{fig:intraday-vol}
\end{figure}

\subsection{Calibrated Simulation Design}
\label{app:calibration:simulation}

We construct a simulation DGP that incorporates all calibrated features while
adding a synthetic treatment assignment for causal inference.

\paragraph{Covariate Process.}
Let $X_t \in \R^p$ follow a VAR(1) process with persistence calibrated
to RX1 spread dynamics:
\begin{equation}
X_t = \Phi X_{t-1} + \eta_t, \quad \eta_t \sim \mathcal{N}(0, \Sigma_\eta),
\end{equation}
where $\Phi = \rho_X I_p$ with $\rho_X = \CovariateRho$ estimated from RX1 spread
autocorrelation.

\paragraph{Treatment Assignment.}
We generate treatment $W_t \in \{0,1\}$ via a logistic propensity model:
\begin{equation}
W_t \mid X_t \sim \text{Bernoulli}(e(X_t)), \quad
e(x) = \text{logit}^{-1}(\gamma_0 + \gamma_1^\top x),
\end{equation}
with $\gamma$ chosen to yield treatment probability $\approx 0.5$.

\paragraph{Potential Outcomes.}
The outcome process incorporates GARCH volatility and heterogeneous treatment
effects:
\begin{align}
Y_t(0) &= \mu_0(X_t) + \sigma_t \varepsilon_t, \\
Y_t(1) &= \mu_1(X_t) + \sigma_t \varepsilon_t, \\
Y_t &= W_t Y_t(1) + (1 - W_t) Y_t(0),
\end{align}
where:
\begin{itemize}
\item $\sigma_t^2$ follows GARCH(1,1) with parameters $(\hat{\omega}, \hat{\alpha}_1, \hat{\beta}_1)$ from Table~\ref{tab:garch-params};
\item $\varepsilon_t \sim t_{\hat{\nu}}$ with $\hat{\nu} = \StudentTDOF$;
\item Intraday seasonality: $\sigma_t \leftarrow \sigma_t \cdot s(h_t)^{1/2}$.
\end{itemize}

\paragraph{Conditional Average Treatment Effect.}
The true CATE is specified as:
\begin{equation}
\tau(x) = \mu_1(x) - \mu_0(x) = \sin(2\pi x_1) + 0.5 x_2,
\end{equation}
providing nonlinear heterogeneity for the conformal inference target.

\subsection{Simulation Parameters}
\label{app:calibration:params}

Table~\ref{tab:sim-params} summarizes all simulation parameters, distinguishing
those calibrated from RX1 data versus standard choices from the literature.

\begin{table}[htbp]
\centering
\caption{Calibrated Simulation Parameters}
\label{tab:sim-params}
\begin{tabular}{llcc}
\toprule
\textbf{Parameter} & \textbf{Description} & \textbf{Value} & \textbf{Source} \\
\midrule
$T$ & Sample size & 2,000 & Standard \\
$p$ & Covariate dimension & 5 & Standard \\
$\rho_X$ & Covariate persistence & \CovariateRho & RX1 spread AC \\
$\omega$ & GARCH intercept & \GARCHOneOmega & RX1 GARCH(1,1) \\
$\alpha_1$ & GARCH ARCH coefficient & \GARCHOneAlpha & RX1 GARCH(1,1) \\
$\beta_1$ & GARCH coefficient & \GARCHOneBeta & RX1 GARCH(1,1) \\
$\nu$ & Student-$t$ degrees of freedom & \StudentTDOF & RX1 residuals \\
$s(h)$ & Intraday seasonality & U-shaped & RX1 pattern \\
$K$ & Cross-fitting blocks & 5 & Standard \\
$\gamma$ & ACI learning rate & 0.005 & Gibbs-Cand\`es \\
$\alpha$ & Nominal miscoverage & 0.10 & Standard \\
\midrule
\multicolumn{4}{l}{\textit{XGBoost hyperparameters (nuisance models)}} \\
--- & \texttt{n\_estimators} & 100 & Default \\
--- & \texttt{max\_depth} & 4 & Tuned \\
--- & \texttt{learning\_rate} & 0.1 & Default \\
--- & \texttt{subsample} & 0.8 & Regularisation \\
\bottomrule
\end{tabular}
\end{table}

The simulations therefore match the temporal dependence structure of real high-frequency data.


\section{Proof Details}
\label{app:proofs}

\subsection{Proof of Lemma~\ref{lem:switch}}
\label{app:proof:lemma8}

The score $s_t = g(Z_t, \ldots, Z_{t-L})$ is a measurable function of $(Z_{t-L}, \ldots, Z_t)$.  By the data-processing inequality for total variation distance, applying a measurable function cannot increase $\dTV$.  Therefore, the switch coefficient of the derived process $S$ is bounded by the switch coefficient of the original process $Z$ with lag adjusted for the memory:
\[
\Psi_{k,\tau}(S) = \dTV(\mathcal{L}(\Delta^0_{k,\tau}(S)), \mathcal{L}(\Delta^1_{k,\tau}(S))) \leq \dTV(\mathcal{L}(\Delta^0_{k,\tau+L}(Z)), \mathcal{L}(\Delta^1_{k,\tau+L}(Z))).
\]
By Proposition~\ref{prop:switchbeta} applied to $Z$ with lag $\tau$, the right-hand side is bounded by $2\beta(\tau - L)$ for $k \leq T - \tau$.  The average bound follows by summing over $k$.

The key subtlety is that the nuisance models $(\hat{e}_k, \hat\mu_k)$ used in constructing $\psi_t^{\mathrm{DR}}$ are trained on out-of-block data $\mathcal{B}_{-k}^g$.  Conditional on the training $\sigma$-field $\mathcal{F}_{\mathrm{train}}$, the nuisance estimates are fixed (non-random) functions, so the score $s_t = |\psi_t^{\mathrm{DR}} - \hat\tau(X_t)|$ is a deterministic function of $(Z_t)$ alone given $\mathcal{F}_{\mathrm{train}}$---that is, $L = 0$ conditionally.  The memory parameter $L$ in the lemma statement arises only if nuisance estimators use lagged covariates $(Z_{t-1}, \ldots, Z_{t-L})$ as features; for block cross-fitting with i.i.d.\ features within each block, $L = 0$ and the bound simplifies to $\Psi_{k,\tau}(S) \leq 2\beta(\tau)$.  The switch coefficient bound then applies conditionally on $\mathcal{F}_{\mathrm{train}}$, and integrating over $\mathcal{F}_{\mathrm{train}}$ preserves the bound.

\subsection{Proof of Lemma~\ref{lem:aci}}
\label{app:proof:lemma11}

Part (i) is the potential-function argument of \citet{gibbs2021aci}: define $\Phi_t = (\alpha_t - \alpha)^2$.  Then $\Phi_{t+1} - \Phi_t = \gamma^2(\alpha - \text{err}_t)^2 + 2\gamma(\alpha_t - \alpha)(\alpha - \text{err}_t)$.  Telescoping gives the result.

The mixing condition enters only if we seek a CLT for the coverage process.  Under Assumption~\ref{ass:mixing}, the process $\text{err}_t - \alpha$ has well-defined long-run variance $\sigma_\infty^2 = \sum_{h=-\infty}^\infty \Cov(\text{err}_0, \text{err}_h)$, which is finite by absolute summability of autocovariances under polynomial $\beta$-mixing \citep{rio2017mixing}.

\subsection{Proof of Corollary~\ref{cor:cate}}
\label{app:proof:cor5}

\textbf{Part (i):}  By the standard DR orthogonality identity \citep[Theorem~1]{chernozhukov2018dml},
\[
\E[\psi_t^{\mathrm{DR}} - \tau(X_t) \mid X_t, \calF_{\mathrm{train}}]
= (\hat\mu_1 - \mu_1)(X_t) \cdot \frac{e(X_t) - \hat{e}(X_t)}{\hat{e}(X_t)}
- (\hat\mu_0 - \mu_0)(X_t) \cdot \frac{e(X_t) - \hat{e}(X_t)}{1 - \hat{e}(X_t)}.
\]
Under the overlap bound (Assumption~\ref{ass:overlap}), $|\hat{e}(X_t)|^{-1}, |1-\hat{e}(X_t)|^{-1} \leq \eta^{-1}$, so the conditional bias is bounded by $2\eta^{-1} \|\hat\mu - \mu\|_\infty \|\hat{e} - e\|_\infty$.

\textbf{Part (ii):}  The algebraic identity \eqref{eq:cate_algebraic} follows directly from the interval construction: $\tau(X_t) \in \hat{C}_t$ if and only if $|\tau(X_t) - \hat\tau(X_t)| \leq \hat{q}_t$, i.e., $\varepsilon_t \leq \hat{q}_t$.

\emph{Calibration target.}  The half-width $\hat{q}_t$ is chosen so that the conformity score $s_t = |\psi_t^{\mathrm{DR}} - \hat\tau(X_t)|$ falls below $\hat{q}_t$ with approximate probability $1-\alpha$.  This calibrates coverage for the pseudo-outcome, not for the latent CATE.

\emph{Asymptotic containment.}  Let $\varepsilon_t = |\hat\tau(X_t) - \tau(X_t)|$ denote the pointwise CATE estimation error.  Under Assumption~\ref{ass:cate}, $\|\hat\tau - \tau\|_{L^2} = o_p(1)$, which implies $\varepsilon_t = o_p(1)$ by Markov's inequality: for any $\delta > 0$, $\Prob(\varepsilon_t > \delta) \leq \delta^{-2} \E[\varepsilon_t^2] = \delta^{-2} \|\hat\tau - \tau\|_{L^2}^2 \to 0$.  The half-width $\hat{q}_t$ tracks the $(1-\alpha)$-quantile of pseudo-outcome deviations, which scales with the irreducible noise $\Var(\xi_t \mid X_t)^{1/2}$.  Since pseudo-outcome variance remains positive, $\hat{q}_t \geq c > 0$ for some constant $c$.  Therefore $\Prob(\varepsilon_t \leq \hat{q}_t) \geq \Prob(\varepsilon_t \leq c) \to 1$.

\textbf{Part (iii):}  When $\varepsilon_t \leq \hat{q}_t$, the latent CATE lies within the interval by construction.  The unconditional probability bound follows from the law of total probability.  This result depends on consistency of $\hat\tau$, not on the conformal calibration targeting CATE error.

\subsection{Aggregation Step for Theorem~\ref{thm:main}}
\label{app:aggregation}

The full triangle inequality decomposition of the coverage gap is:
\begin{align*}
&\left|\frac{1}{T}\sum_{t=1}^T \Prob(\psi_t^{\mathrm{DR}} \notin \hat{C}_t) - \alpha\right| \\
&\leq \underbrace{\left|\frac{1}{T}\sum_{t=1}^T \Prob(\psi_t^{\mathrm{DR}} \notin \hat{C}_t) - \frac{1}{T}\sum_{t=1}^T \Prob(\psi_t^{*} \notin \hat{C}_t^{*})\right|}_{\text{nuisance-bias tax (Lemma~\ref{lem:productbias})}} \\
&\quad + \underbrace{\left|\frac{1}{T}\sum_{t=1}^T \Prob(\psi_t^{*} \notin \hat{C}_t^{*}) - \frac{1}{T}\sum_{t=1}^T \Prob^{\mathrm{exch}}(\psi_t^{*} \notin \hat{C}_t^{*})\right|}_{\text{mixing gap (Lemma~\ref{lem:switch})}} \\
&\quad + \underbrace{\left|\frac{1}{T}\sum_{t=1}^T \Prob^{\mathrm{exch}}(\psi_t^{*} \notin \hat{C}_t^{*}) - \alpha\right|}_{\text{adaptation rate (Lemma~\ref{lem:aci})}},
\end{align*}
where $\psi_t^{*}$ denotes the oracle DR pseudo-outcome with true nuisance functions, $\hat{C}_t^{*}$ the corresponding oracle intervals, and $\Prob^{\mathrm{exch}}$ the coverage under an exchangeable version of the score process.

The first term is bounded by $\|\hat{e}-e\|_2 \cdot \|\hat\mu-\mu\|_2 + O(\beta(g)^{\delta/(2+\delta)})$ by Lemma~\ref{lem:productbias}, where $g$ is the guard band size and $\delta$ is the moment exponent in Assumption~\ref{ass:moment}.  The second is bounded by $\min_\tau\{\tau/T + 2\beta(\tau)\}$ by Lemma~\ref{lem:switch} and the Barber--Pananjady framework.  The third is $O(T^{-1/2})$ by Lemma~\ref{lem:aci}.  With $g \asymp b \asymp T^{1/(r+1)}$, the coupling remainder is $O(T^{-r/(r+1)}) = o(T^{-1/2})$ under $r > 1$, and combining gives the main bound.

\subsection{Guard Band Sensitivity}\label{app:guardsens}

\begin{table}[h]
\centering
\caption{Guard-band sensitivity: VS-DR-ACI at $\rho=0.95$, $T=2000$, 100 replications.  Effective calibration size decreases as $g$ grows.}\label{tab:guardsens}
\small
\begin{tabular}{@{}lccc@{}}
\toprule
Guard band $g$ & Coverage gap (SE) & Median width & Cal.\ size \\
\midrule
$b/2 = 200$ & $+$0.000\,{\scriptsize(0.000)} & 9.5 & 1200 \\
$b = 400$ & $+$0.000\,{\scriptsize(0.000)} & 9.5 & 1000 \\
$2b = 800$ & $+$0.001\,{\scriptsize(0.001)} & 11.3 & 800 \\
\bottomrule
\end{tabular}
\end{table}

Coverage remains valid throughout; interval width increases modestly with $g$ as expected from the reduced effective calibration sample.

\subsection{Stability of VS-DR-ACI under Covariate Shift}
\label{app:vsstability}

\begin{proposition}[Bounded VS scores under overlap]\label{prop:vsstability}
Suppose Assumption~\ref{ass:overlap} holds with overlap constant $\eta > 0$, outcomes are bounded ($|Y| \leq B$ a.s.), and $c_0 = \min_x \min_w \Var(Y \mid X=x, W=w) > 0$.  Then the VS conformity scores satisfy
\[
\sup_t s_t^{\mathrm{VS}} \leq \frac{2B\sqrt{2/\eta}}{\sqrt{c_0/\eta}} = \frac{2B\sqrt{2}}{\sqrt{c_0}} < \infty,
\]
and consequently the ACI quantiles $\hat{q}_t$ remain bounded regardless of drift in $P(X_t)$.
\end{proposition}

\begin{proof}
Under Assumption~\ref{ass:overlap}, $\eta \leq e(x) \leq 1-\eta$ uniformly.  Substituting into \eqref{eq:noisevar}:
\[
\sigma_\xi^2(x) = \frac{\Var(Y \mid X=x, W=1)}{e(x)} + \frac{\Var(Y \mid X=x, W=0)}{1-e(x)} \geq \frac{c_0}{\eta},
\]
so $\hat\sigma_\xi(X_t) \geq c := \sqrt{c_0/\eta} > 0$ uniformly.  The unscaled score satisfies $s_t = |\psi_t^{\mathrm{DR}} - \hat\tau(X_t)| \leq 2B\sqrt{2/\eta}$ under bounded outcomes and overlap.  Hence $s_t^{\mathrm{VS}} = s_t/\hat\sigma_\xi(X_t) \leq M/c < \infty$.  Bounded scores imply bounded ACI quantiles $\hat{q}_t \in [0, M/c]$, so the ACI update \eqref{eq:aci} cannot spiral regardless of drift in $P(X_t)$.  This explains why VS-DR-ACI survives Regime~D (Section~\ref{sec:regimed}) while unscaled DR-ACI spirals.  The argument requires that drift preserves the support of $X_t$ so that Assumption~\ref{ass:overlap} continues to hold.
\end{proof}


\section{Application Robustness}
\label{app:robustness}

This appendix provides full heterogeneity and robustness analysis for the Dynamic M-ELO application.

\subsection{Pre-Treatment Balance and Panel Dimensions}\label{app:balance}

\begin{table}[ht]
\centering
\caption{Dynamic M-ELO: Pre-Treatment Balance and Panel Dimensions}\label{tab:melo-balance}
\begin{tabular}{lcccc}
\toprule
& Early Cohorts & Late Cohorts & Diff & $p$-val \\
& (W--Z, T--V) & (M--S, A--L) & & \\
\midrule
\multicolumn{5}{l}{\textit{Panel A: Pre-Treatment Daily Outcome Means}} \\[3pt]
Hidden share & 0.281 & 0.293 & $-0.013$ & $<0.001$ \\
Odd-lot share & 0.601 & 0.601 & 0.000 & 0.607 \\
$\log(1+\text{LitTrades})$ & 6.541 & 6.252 & 0.289 & $<0.001$ \\
$\log(1+\text{Trades})$ & 6.958 & 6.699 & 0.259 & $<0.001$ \\[6pt]
\multicolumn{5}{l}{\textit{Panel B: Panel Dimensions}} \\[3pt]
Daily obs.\ (ticker $\times$ day) & \multicolumn{2}{c}{4,640,505} & & \\
$N$ tickers & \multicolumn{2}{c}{9,691} & & \\
Sample period & \multicolumn{2}{c}{Jan 2023 -- Jun 2025} & & \\
\bottomrule
\end{tabular}
\end{table}

The statistically significant differences in daily trades, volume, and hidden share ($p < 0.01$) reflect the larger share of high-activity tickers with names starting T--Z; these level differences are absorbed by symbol fixed effects in the TWFE specification.  Pre-treatment hidden share is 1.2 percentage points lower for early cohorts ($p < 0.001$), reinforcing the importance of parallel-trends validation.

\subsection{Heterogeneity Analysis}\label{app:heterogeneity}

We stratify CATEs by pre-treatment M-ELO intensity (quartiles based on baseline hidden share):

\begin{table}[ht]
\centering
\caption{Heterogeneous Effects by Pre-Treatment Intensity}\label{tab:melo-heterogeneity}
\begin{tabular}{lccc}
\toprule
Intensity Quartile & Hidden Share ATT & SE & $p$-value \\
\midrule
Q1 (Low, median 15\%) & $+0.049$ & 0.012 & $<0.001$ \\
Q2 (median 22\%) & $+0.060$ & 0.014 & $<0.001$ \\
Q3 (median 31\%) & $+0.020$ & 0.011 & 0.072 \\
Q4 (High, median 45\%) & $-0.007$ & 0.015 & 0.647 \\
\bottomrule
\end{tabular}
\end{table}

The dose-response pattern is monotonically decreasing: tickers with the least prior midpoint usage gain the most from Dynamic M-ELO.  High-intensity tickers (Q4) show no significant effect---they were already utilising midpoint orders extensively, so the faster timer provided no marginal benefit.

\textbf{Volume quintile analysis.}  Stratifying by pre-treatment trading volume reveals a complementary pattern: low-volume stocks exhibit the largest positive CATEs ($+0.036$, $p < 0.001$ for Q1), while high-volume stocks show null or slightly negative effects ($-0.001$, $p = 0.65$ for Q5).  This confirms that Dynamic M-ELO benefits primarily accrue to less liquid securities where midpoint execution was previously scarce.

\subsection{Robustness Tests}\label{app:robustness-tests}

\textbf{Placebo test.}  Shifting the adoption date 180 days earlier yields a placebo ATT of $-0.018$ ($p = 0.06$) for hidden share---opposite sign from the true effect, providing reassurance against confounding pre-trends.  Odd-lot share shows a similarly clean null placebo.  However, lit trades exhibits a significant positive placebo ($\beta = +0.056$, $p < 0.001$), indicating possible confounding pre-trends; the lit trades result should be interpreted with caution.

\textbf{Sample restrictions.}  Effects are robust to dropping test symbols, excluding the first 5 post-adoption days, restricting to Nasdaq-listed securities only, and using a tight window (June 2023--March 2025).

\textbf{Winsorisation.}  Results are unchanged under 1\% and 5\% winsorisation of outcome variables.

\textbf{Multiple testing.}  Hidden share and odd-lot share survive both Bonferroni ($\alpha/7 = 0.007$) and Benjamini--Hochberg FDR corrections at the 5\% level.  Lit trades survives multiple testing corrections but is subject to the pre-trend caveat noted above.

\subsection{Sensitivity Analysis}\label{app:sensitivity}

Unconfoundedness (Assumption~\ref{ass:unconf}) is untestable.  Robustness can be assessed using the omitted variable bias framework of \citet{cinelli2020sensitivity}, reporting E-values \citep{vanderweele2017evalue} that quantify the minimum confounder strength needed to invalidate coverage.

\paragraph{Code Availability.}
Code for DR-ACI and all experiments is available at \url{https://github.com/rockandrolla13/draci}.

\bibliography{references}

\end{document}